%% file: main.tex
\documentclass{article}

\PassOptionsToPackage{numbers}{natbib}
\usepackage[final]{neurips_2019}

\usepackage[utf8]{inputenc} 
\usepackage[T1]{fontenc}    
\usepackage{hyperref}       
\usepackage{url}            
\usepackage{booktabs}       
\usepackage{amsfonts}       
\usepackage{nicefrac}       
\usepackage{microtype}      
\usepackage{multirow, multicol}
\usepackage{amsmath, amsthm}
\usepackage{thmtools, thm-restate}
\usepackage{bbm}
\usepackage{caption}
\usepackage{color, graphicx}
\usepackage{enumitem}

\usepackage{dcolumn}
\newcolumntype{d}[1]{D{.}{.}{#1}}

\usepackage{floatrow}
\newfloatcommand{capbtabbox}{table}[][\FBwidth]

\usepackage{letltxmacro}

\title{Don't Blame the ELBO!\\A Linear VAE Perspective on Posterior Collapse}

\author{%
  James Lucas$^{\ddagger}$\thanks{Intern at Google Brain},~~George Tucker$^{\dagger}$,~~Roger Grosse$^{\ddagger}$,~~Mohammad Norouzi$^{\dagger}$\\ \\
  ~~$\ddagger$University of Toronto\hspace{1.5cm}$\dagger$Google Brain
}

\input{defs.tex}
\input{macros.tex}

\begin{document}

\maketitle

\input{abstract.tex}
\input{introduction.tex}
\input{background.tex}
\input{related.tex}
\input{linear_vae.tex}
\input{deep_gaussian_vae.tex}
\input{experiments/experiments.tex}

\input{conclusion.tex}
\input{acknowledgements.tex}
\newpage
\bibliography{references}
\bibliographystyle{abbrvnat}

\newpage
\input{appendix/appendices.tex}

\end{document}

%% file: defs.tex
\newcommand{\bz}{\mathbf{z}}
\newcommand{\bx}{\mathbf{x}}
\newcommand{\by}{\mathbf{y}}
\newcommand{\bu}{\mathbf{u}}

\newcommand{\bI}{\mathbf{I}}
\newcommand{\bW}{\mathbf{W}}
\newcommand{\bD}{\mathbf{D}}
\newcommand{\bC}{\mathbf{C}}
\newcommand{\bS}{\mathbf{S}}
\newcommand{\bV}{\mathbf{V}}
\newcommand{\bL}{\mathbf{L}}
\newcommand{\bU}{\mathbf{U}}
\newcommand{\bR}{\mathbf{R}}
\newcommand{\bM}{\mathbf{M}}

\newcommand{\bLambda}{\boldsymbol{\Lambda}}
\newcommand{\bmu}{\boldsymbol{\mu}}

\newcommand{\cN}{\mathcal{N}}
\newcommand{\reals}{\mathbb{R}}
\newcommand{\expect}{\mathbb{E}}

\newcommand{\bbP}{\mathbb{P}}

\def\bWmle{\bW_{\mathrm{\!MLE}}}
\def\sigmamle{\sigma_{\mathrm{MLE}}}

%% file: macros.tex
\newif\ifcomments
\commentsfalse

\definecolor{orange}{rgb}{0.968, 0.545, 0.176}

\newcommand\blfootnote[1]{%
  \begingroup
  \renewcommand\thefootnote{}\footnote{#1}%
  \addtocounter{footnote}{-1}%
  \endgroup
}

\def\ie{{\em i.e.,}}
\LetLtxMacro{\originaleqref}{\eqref}
\renewcommand{\eqref}{Eq.~\originaleqref}

\newcommand{\mc}[1]{\multicolumn{1}{|c|}{#1}}

%% file: abstract.tex
\begin{abstract}
\vspace{-.2cm}
    Posterior collapse in Variational Autoencoders (VAEs) arises when the variational posterior distribution closely matches the prior for a subset of latent variables.
    This paper presents a simple and intuitive explanation for posterior collapse through the analysis of linear VAEs and their direct correspondence with Probabilistic PCA (pPCA).
    We explain how posterior collapse may occur in pPCA due to local maxima in the log marginal likelihood. Unexpectedly, we prove that the ELBO objective for the linear VAE does not introduce additional spurious local maxima relative to log marginal likelihood. 
    We show further that training a linear VAE with exact variational inference recovers an identifiable global maximum corresponding to the principal component directions. 
    Empirically, we find that our linear analysis is predictive even for high-capacity, non-linear VAEs and helps
    explain the relationship between the observation noise, local maxima, and posterior collapse in deep Gaussian VAEs.
\end{abstract}

%% file: introduction.tex
\section{Introduction}
\vspace{-.2cm}

The generative process of a deep latent variable model entails drawing a number of latent factors from the prior and using a neural network to convert such factors to real data points.
Maximum likelihood estimation of the parameters requires marginalizing out the latent factors, which is intractable for deep latent variable models.
The influential work of \citet{kingma2013auto} and \citet{rezende2014stochastic} on Variational Autoencoders (VAEs) enables optimization of a tractable lower bound on the likelihood via
a reparameterization of the Evidence Lower Bound (ELBO)~\citep{jordan1999introduction, blei2017variational}.
This has led to a surge of recent interest in automatic discovery of the latent factors of variation for a data distribution based on VAEs and principled probabilistic modeling~\citep{higgins2016beta, bowman2015generating, chen2018isolating, molauto18}. \blfootnote{Code available at \href{https://sites.google.com/view/dont-blame-the-elbo}{https://sites.google.com/view/dont-blame-the-elbo}}

Unfortunately, the quality and the number of the latent factors learned is influenced by a phenomenon known as {\em posterior collapse},
where the generative model learns to ignore a subset of the latent variables.
Most existing papers suggest that posterior collapse is caused by the KL-divergence term in the ELBO objective, which directly encourages the variational distribution to match the prior~\citep{bowman2015generating, kingma2016improved, sonderby2016ladder}.
Thus, a wide range of heuristic approaches in the literature have attempted to diminish the effect of the KL term in the ELBO to alleviate posterior collapse \citep{bowman2015generating, razavi2018preventing, sonderby2016ladder, huang2018improving}.
While holding the KL term responsible for posterior collapse makes intuitive sense, the mathematical mechanism of this phenomenon is not well understood.
In this paper, we investigate the connection between posterior collapse and spurious local maxima in the ELBO objective
through the analysis of linear VAEs.
Unexpectedly, we show that spurious local maxima may arise even in the optimization of exact
marginal likelihood, and such local maxima are linked with a collapsed posterior.

While linear autoencoders \citep{rumelhart1985learning} have been studied extensively \citep{baldi1989neural, kunin2019loss},
little attention has been given to their variational counterpart from a theoretical standpoint.
A well-known relationship exists between linear autoencoders and PCA --
the optimal solution 
of a linear autoencoder has decoder weight columns that span the same subspace as the one defined by the principal components \citep{baldi1989neural}.
Similarly, the maximum likelihood solution of probabilistic PCA (pPCA)~\citep{tipping1999probabilistic} 
recovers the subspace of principal components.
In this work, we show that a linear variational autoencoder can recover the solution of pPCA.
In particular, by specifying a diagonal covariance structure on the variational distribution,
one can recover an identifiable autoencoder, which at the global maximum of the ELBO recovers the exact principal components
as the columns of the decoder's weights.
Importantly, we show that the ELBO objective for a linear VAE does not introduce any local maxima beyond the log marginal likelihood.

The study of linear VAEs gives us new insights into the cause of posterior collapse and the difficulty of VAE optimization more generally.
Following the analysis of \citet{tipping1999probabilistic}, we characterize the stationary points of pPCA and show that the {\em variance of the observation model} directly influences the stability of local stationary points corresponding to posterior collapse
-- it is only possible to escape these sub-optimal solutions by simultaneously reducing noise and learning better features. Our contributions include:
\begin{itemize}[topsep=0pt, partopsep=0pt, leftmargin=25pt, parsep=0pt, itemsep=2pt]
    \item We verify that linear VAEs can recover the true posterior of pPCA. Further, we prove that the global optimum of the linear VAE recovers the principal components (not just their spanning sub-space). More importantly, we prove that using ELBO to train linear VAEs does not introduce any additional spurious local maxima relative to log marginal likelihood training.
    \item While high-capacity decoders are often blamed for posterior collapse, we show that posterior collapse may occur when optimizing log marginal likelihood even without powerful decoders. Our experiments verify the analysis of the linear setting and show that these insights extend even to high-capacity non-linear VAEs. Specifically, we provide evidence that the observation noise in deep Gaussian VAEs plays a crucial role in overcoming local maxima corresponding to posterior collapse.
\end{itemize}

%% file: background.tex
\section{Preliminaries}\label{sec:prelim}

\paragraph{Probabilistic PCA.}~The probabilitic PCA (pPCA) model is defined as follows. Suppose latent variables $\bz \in \reals^k$ generate data $\bx \in \reals^n$. A standard Gaussian prior is used for $\bz$ and a linear generative model with a spherical Gaussian observation model for $\bx$:%
\begin{align}
    \begin{split}\label{eqn:ppca}
        p(\bz) &~=~ \cN(\mathbf{0},\bI)~, \\
        p(\bx \mid \bz) &~=~ \cN( \bW \bz + \bmu, \sigma^2 \bI)~.
    \end{split}
\end{align}

The pPCA model is a special case of factor analysis \citep{bartholomew1987latent}, which uses a spherical covariance $\sigma^2\bI$ instead of a full covariance matrix.
As pPCA is fully Gaussian, both the marginal distribution for $\bx$ and the posterior $p(\bz\mid\bx)$ are Gaussian, and unlike factor analysis, the maximum likelihood estimates of $\bW$ and $\sigma^2$ are tractable~\citep{tipping1999probabilistic}.

\vspace{-.1cm}
\paragraph{Variational Autoencoders.}~Recently, amortized variational inference has gained popularity as a means to learn complicated latent variable models. In these models, the log marginal likelihood, $\log p(\bx)$, is intractable but a variational distribution, denoted $q(\bz\!\mid\!\bx)$, is used to approximate the posterior $p(\bz\!\mid\!\bx)$, allowing tractable approximate inference using the Evidence Lower Bound (ELBO):
\begin{eqnarray}
\log p(\bx) &=& \expect_{q(\bz \mid \bx)}[\log p(\bx,\bz) -\log q(\bz \mid \bx)] + D_{KL}(q(\bz \mid \bx) || p(\bz \mid \bx)) \\
&\geq& \expect_{q(\bz \mid \bx)}[\log p(\bx,\bz)-\log q(\bz \mid \bx)] \\
&=& \expect_{q(\bz \mid \bx)}[\log p(\bx \mid \bz)] -D_{KL}(q(\bz \mid \bx) || p(\bz))  \qquad (:= ELBO) \label{eqn:elbo}
\end{eqnarray}

The ELBO~\citep{jordan1999introduction, blei2017variational} consists of two terms, the KL divergence between the variational distribution, $q(\bz|\bx)$, and prior, $p(\bz)$, and the expected conditional log-likelihood.
The KL divergence forces the variational distribution towards the prior and so has reasonably been the focus of many attempts to alleviate posterior collapse.
We hypothesize that the log marginal likelihood itself often encourages posterior collapse.

In Variational Autoencoders (VAEs), two neural networks are used to parameterize  $q_\phi(\bz|\bx)$ and $p_\theta(\bx|\bz)$, where $\phi$ and $\theta$ denote two sets of neural network weights.
The encoder maps an input $\bx$ to the parameters of the variational distribution, and then the decoder maps a sample from the variational distribution back to the inputs.

\vspace{-.1cm}
\paragraph{Posterior collapse.} A dominant issue with VAE optimization is posterior collapse, in which the learned variational distribution is close to the prior. This reduces the capacity of the generative model, making it impossible for the decoder network to make use of the information content of all of the latent dimensions. While posterior collapse is widely acknowledged, formally defining it has remained a challenge. We introduce a formal definition in Section~\ref{sec:experiments:posterior_collapse} which we use to measure posterior collapse in trained deep neural networks.

%% file: related.tex
\begin{figure}
    \small
    \centering
    \begin{minipage}{0.31\linewidth}
    \centering
    \includegraphics[width=0.84\linewidth]{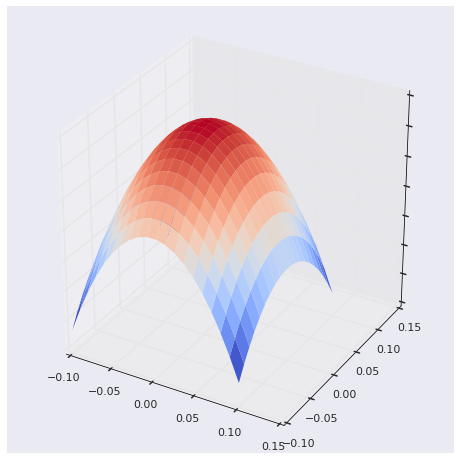}\\
    a) $\sigma^2 = \lambda_4$
    \end{minipage}\hfill
    \begin{minipage}{0.31\linewidth}
    \centering
    \includegraphics[width=.84\linewidth]{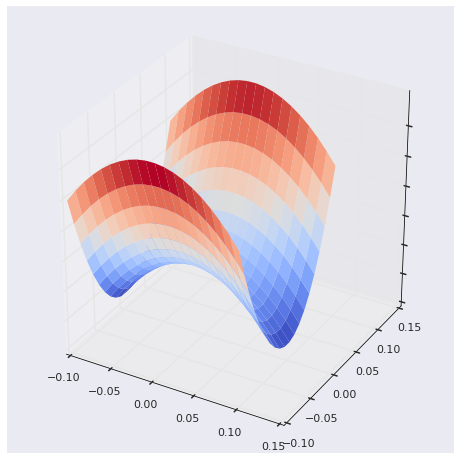}\\
    b) $\sigma^2 = \lambda_6$
    \end{minipage}\hfill
    \begin{minipage}{0.31\linewidth}
    \centering
    \includegraphics[width=.84\linewidth]{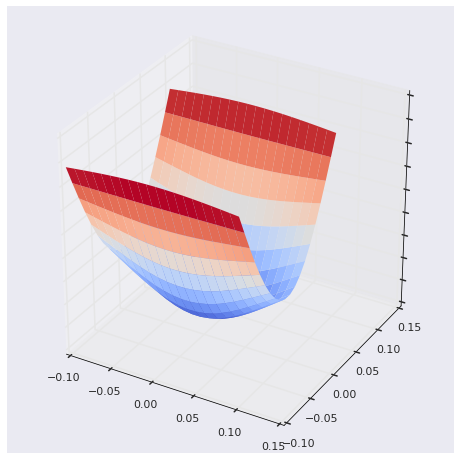}\\
    c) $\sigma^2 = \lambda_8$
    \end{minipage}
    \caption{\small \textbf{Stationary points of pPCA.} Two zero-columns of $\bW$ are perturbed in the directions of two orthogonal principal components ($\mu_5$ and $\mu_7$) and
    the optimization landscape around zero-columns is shown, where the goal is to maximize log marginal likelihood. The stability of the stationary points depends critically on $\sigma^2$ (the observation noise). Left: $\sigma^2$ is too large to capture either principal component. Middle: $\sigma^2$ is too large to capture one of the principal components. Right: $\sigma^2$ is able to capture both principal components.
    \vspace{-0.4cm}
    }
    \label{fig:ppca_stationary}
\end{figure}

\section{Related Work}
\vspace{-0.2cm}

\citet{dai2017hidden} discuss the relationship between robust PCA methods~\citep{candes2011robust} and VAEs.
They show that at stationary points the VAE objective locally aligns with pPCA under certain assumptions. We study the pPCA objective explicitly and show a direct correspondence with linear VAEs.
\citet{dai2017hidden} showed that the covariance structure of the variational distribution may smooth out the loss landscape.
This is an interesting result whose interactions with ours is an exciting direction for future research.

\citet{he2018lagging} motivate posterior collapse through an investigation of the learning dynamics of deep VAEs.
They suggest that posterior collapse is caused by the inference network lagging behind the true posterior during the early stages of training.
A related line of research studies issues arising from approximate inference causing a mismatch between the variational distribution and true posterior \citep{cremer2018inference, kim2018semi, hjelm2016iterative}.
By contrast, we show that posterior collapse may exist even when the variational distribution matches the true posterior exactly. 

\citet{alemi2017fixing} used an information theoretic framework to study the representational properties of VAEs. They show that with infinite model capacity there are solutions with equal ELBO and log marginal likelihood which span a range of representations, including posterior collapse.
We find that even with weak (linear) decoders, posterior collapse may occur. Moreover, we show that in the linear case this posterior collapse is due entirely to the log marginal likelihood.

The most common approach for dealing with posterior collapse is to anneal a weight on the KL term during training from $0$ to $1$ \citep{bowman2015generating, sonderby2016ladder, maaloe2019biva, higgins2016beta, huang2018improving}.
Unfortunately, this means that during the annealing process, one is no longer optimizing a bound on the log-likelihood. Also, it is difficult to design these annealing schedules and we have found that once regular ELBO training resumes the posterior will typically collapse again (Section~\ref{sec:post_collapse_mnist_vae}).

\citet{kingma2016improved} propose a constraint on the KL term, termed "free-bits", where the gradient of the KL term per dimension is ignored if the KL is below a given threshold.
Unfortunately, this method reportedly has some negative effects on training stability \citep{razavi2018preventing, chen2016variational}.
Delta-VAEs \citep{razavi2018preventing} instead choose prior and variational distributions such that the variational distribution can never exactly recover the prior, allocating free-bits implicitly. Several other papers have studied alternative formulations of the VAE objective \citep{rezende2018taming, dai2018diagnosing, alemi2017fixing, ma2018mae, yeung2017tackling}.
\citet{dai2018diagnosing} analyzed the VAE objective to improve image fidelity under Gaussian observation models and also discuss the importance of the observation noise. Other approaches have explored changing the VAE network architecture to help alleviate posterior collapse; for example adding skip connections \citep{maaloe2019biva, dieng2018avoiding}

\citet{rolinek2018variational} observed that the diagonal covariance used in the variational distribution of VAEs encourages orthogonal representations.
They use linearizations of deep networks to prove their results under a modification of the objective function by explicitly ignoring latent dimensions with posterior collapse.
Our formulation is distinct in focusing on linear VAEs without modifying the objective function and proving an exact correspondence between the global solution of linear VAEs and the principal components.

\citet{kunin2019loss} studied the optimization challenges in the linear autoencoder setting. They exposed an equivalence between pPCA and Bayesian autoencoders and point out that when $\sigma^2$ is too large information about the latent code is lost. A similar phenomenon is discussed in the supervised learning setting by \citet{chechik2005information}. \citet{kunin2019loss} also showed that suitable regularization allows the linear autoencoder to recover the principal components up to rotations. We show that linear VAEs with a diagonal covariance structure recover the principal components \emph{exactly}.

%% file: linear_vae.tex
\section{Analysis of linear VAE}\label{sec:linear_vae}
\vspace*{-.2cm}

This section compares and analyzes the loss landscapes of both pPCA and linear variational autoencoders. 
We first discuss the stationary points of pPCA and then show that a simple linear VAE can recover the global optimum of pPCA. Moreover, when the data covariance eigenvalues are distinct, the linear VAE identifies the individual principal components, unlike pPCA, which recovers only the PCA subspace. Finally, we prove that ELBO does not introduce any additional spurious maxima to the loss landscape.

\subsection{Probabilistic PCA Revisited}
\vspace*{-.1cm}

The pPCA model (\eqref{eqn:ppca}) is a fully Gaussian linear model, thus we can compute both the marginal distribution for $\bx$ and the posterior $p(\bz \mid \bx)$ in closed form:
\begin{eqnarray}
    p(\bx) &=& \cN(\bmu, \bW\bW^\top + \sigma^2 \bI), \label{eqn:ppca_marginal} \\
    p(\bz \mid \bx) &=& \cN(\bM^{-1}\bW^\top(\bx - \bmu), \sigma^2\bM^{-1}), \label{eqn:ppca_posterior}
\end{eqnarray}

where $\bM = \bW^\top\bW + \sigma^2 \bI$. This model is particularly interesting to analyze in the setting of variational inference, as the ELBO can also be computed in closed form (see Appendix~\ref{app:elbo_stationary}).

\vspace*{-.1cm}
\paragraph{Stationary points of pPCA} We now characterize the stationary points of pPCA, largely repeating the thorough analysis of \citet{tipping1999probabilistic} (see Appendix~A of their paper). The maximum likelihood estimate of $\bmu$ is the mean of the data. We can compute $\bWmle$ and $\sigmamle^2$ as follows:%
\begin{eqnarray}
    \sigmamle^2 &=& \frac{1}{n - k} \sum_{j=k+1}^{n}\lambda_j,\\
    \bWmle &=& \bU_{k}(\bLambda_k - \sigmamle^2 \bI)^{1/2}\bR.\label{eqn:ppca_wmle}
\end{eqnarray}

Here $\bU_k$ corresponds to the first $k$ principal components of the data with the corresponding eigenvalues $\lambda_1, \ldots, \lambda_k$ stored in the $k\times k$ diagonal matrix $\bLambda_k$. The matrix $\bR$ is an arbitrary rotation matrix which accounts for weak identifiability in the model. We can interpret $\sigma^2_{MLE}$ as the average variance lost in the projection. The MLE solution is the global optimum. Other stationary points correspond to zeroing out columns of $\bWmle$ (posterior collapse).

\vspace*{-.1cm}
\paragraph{Stability of $\bWmle$} In this section we consider $\sigma^2$ to be fixed and not necessarily equal to the MLE solution. Equation~\ref{eqn:ppca_wmle} remains a stationary point when the general $\sigma^2$ is swapped in. One surprising observation is that $\sigma^2$ directly controls the stability of the stationary points of the log marginal likelihood (see Appendix~\ref{app:ppca_stationary}). In Figure~\ref{fig:ppca_stationary}, we illustrate one such stationary point of pPCA for different values of $\sigma^2$. We computed this stationary point by taking $\bW$ to have three principal component columns and zeros elsewhere. Each plot shows the same stationary point perturbed by two orthogonal vectors corresponding to other principal components.

The stability of the pPCA stationary points depends on the size of $\sigma^2$ --- as $\sigma^2$ increases the stationary point tends towards a stable local maximum so that we cannot learn the additional components. Intuitively, the model prefers to explain deviations in the data with the larger observation noise. Fortunately, decreasing $\sigma^2$ will increase likelihood at these stationary points so that when learning $\sigma^2$ simultaneously these stationary points are saddle points \citep{tipping1999probabilistic}. Therefore, learning $\sigma^2$ is necessary for gaining a full latent representation.

\subsection{Linear VAEs recover pPCA}
\vspace*{-.1cm}

We now show that linear VAEs can recover the globally optimal solution to Probabilistic PCA. We will consider the following VAE model,
\begin{align}
    \begin{split}\label{eqn:linear_vae}
        p(\bx \mid \bz) ~=~ \cN( \bW \bz + \bmu, \sigma^2 \bI), \\
        q(\bz \mid \bx) ~=~ \cN(\bV (\bx - \bmu), \bD),
    \end{split}
\end{align}
where $\bD$ is a diagonal covariance matrix, used globally for all of the data points. While this is a significant restriction compared to typical VAE architectures, which define
an amortized variance for each input point, this is sufficient to recover the global optimum of the probabilistic model.

\begin{restatable}[]{lemma}{vaeglobalpca}
\label{lemma:vae_global_ppca}
The global maximum of the ELBO objective (\eqref{eqn:elbo}) for the linear VAE (\eqref{eqn:linear_vae}) is identical to the global maximum for the log marginal likelihood of pPCA (\eqref{eqn:ppca_marginal}).
\end{restatable}
\vspace{-.4cm}
\begin{proof}
Note that the global optimum of pPCA is defined
up to an orthogonal transformation of the columns of $\bW$, \ie~any rotation $\bR$ in \eqref{eqn:ppca_wmle} results in a matrix $\bWmle$ that given $\sigmamle^2$
attains maximum marginal likelihood.
The linear VAE model defined in \eqref{eqn:linear_vae} is able to recover the global optimum of pPCA 
when $\bR=\bI$.
Recall from \eqref{eqn:ppca_posterior} that $p(\bz \mid \bx)$ is defined in terms of
$\bM = \bW^\top\bW + \sigma^2 \bI$.
When $\bR=\bI$, we obtain $\bM = \bWmle^\top\bWmle + \sigmamle^2 \bI= \bLambda_k$, which is diagonal.
Thus, setting $\bV = \bM^{-1}\bWmle^\top$ and $\bD=\sigmamle^2\bM^{-1} = \sigmamle^2\bLambda_k^{-1}$, recovers the true posterior with diagonal covariance at the global optimum.
In this case, the ELBO equals the log marginal likelihood and is maximized when the decoder has weights $\bW = \bWmle$. Because the ELBO lower bounds log-likelihood, the global maximum of the ELBO for the linear VAE
is the same as the global maximum of the marginal likelihood for pPCA.
\end{proof}

The result of Lemma~\ref{lemma:vae_global_ppca} is somewhat expected because the posterior of pPCA is Gaussian. Further details are given in Appendix~\ref{app:elbo_stationary}.
In addition, we prove a more surprising result that suggests restricting the variational distribution to a Gaussian with a diagonal covariance structure allows one to {\em identify} the principal components at the global optimum of ELBO.

\begin{restatable}[]{corollary}{vaegetspca}
\label{corollary:vae_gets_pca}
The global maximum of the ELBO objective (\eqref{eqn:elbo}) for the linear VAE (\eqref{eqn:linear_vae}) has the scaled principal components as the columns of the decoder network.
\end{restatable}
\vspace{-.4cm}
\begin{proof}
Follows directly from the proof of Lemma~\ref{lemma:vae_global_ppca} and \eqref{eqn:ppca_wmle}.
\end{proof}

We discuss this result in Appendix~\ref{app:vae_identifiability}. This full identifiability is non-trivial and is not achieved even with the regularized linear autoencoder \citep{kunin2019loss}.

So far, we have shown that at its global optimum the linear VAE recovers the pPCA solution, which enforces orthogonality of the decoder weight columns. However, the VAE is trained with the ELBO rather than the log marginal likelihood --- often using SGD. The majority of existing work suggests that the KL term in the ELBO objective is responsible for posterior collapse. So, we should ask whether this term introduces additional spurious local maxima. Surprisingly, for the linear VAE model the ELBO objective \emph{does not} introduce any additional spurious local maxima. We provide a sketch of the proof below with full details in Appendix~\ref{app:elbo_stationary}.

\begin{restatable}[]{theorem}{elbomaxima}
\label{thm:no_elbo_maxima}
The ELBO objective for a linear VAE does not introduce any additional local maxima to the pPCA model.
\end{restatable}
\vspace{-.4cm}
\begin{proof}(Sketch)
If the decoder has orthogonal columns, then the variational distribution recovers the true posterior at stationary points. Thus, the variational objective will exactly recover the log marginal likelihood.
If the decoder does not have orthogonal columns then the variational distribution is no longer tight. However, the ELBO can always be increased by applying an infinitesimal rotation to the right-singular vectors of the decoder towards identity: $\bW' \leftarrow \bW \bR_{\epsilon}$ (so that the decoder columns are closer to orthogonal).
This works because the variational distribution can fit the posterior more closely while the log marginal likelihood is invariant to rotations of the weight columns.
Thus, any additional stationary points in the ELBO objective must necessarily be saddle points.
\end{proof}
The theoretical results presented in this section provide new intuition for posterior collapse in VAEs. In particular, the KL between the variational distribution and the prior is not entirely responsible for posterior collapse --- log marginal likelihood has a role. The evidence for this is two-fold. We have shown that log marginal likelihood may have spurious local maxima but also that in the linear case the ELBO objective does not add any additional spurious local maxima. Rephrased, in the linear setting the problem lies entirely with the probabilistic model. We should then ask, to what extent do these results hold in the non-linear setting?

%% file: deep_gaussian_vae.tex
\section{Deep Gaussian VAEs}

The deep Gaussian VAE consists of a decoder $D_\theta$ and an encoder $E_\phi$. The ELBO objective can be expressed as,
\begin{equation}\label{eqn:gaussian_elbo}
    \mathcal{L}(\bx; \theta, \phi) ~=~ -\mathrm{KL}(q_\phi(\bz\mid\bx)\:\Vert\: p(\bz)) - \frac{1}{2\sigma^2}\expect_{q_\phi(\bz|\bx)}\left[\lVert D_\theta(\bz) - \bx \rVert^2 \right] - \frac{1}{2}\log(2\pi\sigma^2)
\end{equation}
The role of $\sigma^2$ in this objective invites a natural comparison to the $\beta$-VAE objective \citep{higgins2016beta},
where the KL term is weighted by $\beta \in \mathbb{R}^+$. 
\citet{alemi2017fixing} propose using small $\beta$ values to force powerful decoders to utilize the latent variables, but this comes at the cost of poor ELBO. Practitioners must then use downstream task performance for model selection, thus sacrificing one of the primary benefits of likelihood-based models.
However, for a given $\beta$, one can find a corresponding $\sigma^2$ (and a learning rate) such that the gradient updates to the network parameters are identical. Importantly, the Gaussian partition function for a Gaussian observation model (the last term on the RHS of \eqref{eqn:gaussian_elbo}) prevents ELBO from 
deviating from the $\beta$-VAE's objective with a $\beta$-weighted KL term while maintaining the benefits to representation learning when $\sigma^2$ is small. For the Gaussian VAE, this helps connect the dots between the role of local maxima and observation noise in posterior collapse {\em vs.}
heuristic approaches that attempted to alleviate posterior collapse by diminishing the effect of the KL term~\citep{bowman2015generating, razavi2018preventing, sonderby2016ladder, huang2018improving}. In the following section, we will study the nonlinear VAE empirically and explore connections to the linear theory.

%% file: experiments/experiments.tex
\section{Experiments}\label{sec:experiments}
\vspace{-.2cm}

In this section, we present empirical evidence found from studying two distinct claims. First, we verify our theoretical analysis of the linear VAE model. Second, we explore to what extent these insights apply to deep nonlinear VAEs.

\subsection{Linear VAEs}
\vspace{-.1cm}
\input{experiments/linear_vae.tex}

\subsection{Investigating posterior collapse in deep nonlinear VAEs}
\label{sec:experiments:posterior_collapse}
\vspace{-.1cm}
\label{sec:post_collapse_mnist_vae}

\input{experiments/mnist_vaes.tex}

\input{experiments/deeper_vaes.tex}

%% file: experiments/linear_vae.tex
We ran two sets of experiments on 1000 randomly chosen MNIST images. First, we trained linear VAEs with learnable $\sigma^2$ for a range of hidden dimensions\footnote{The VAEs were trained using the analytic ELBO (Appendix~\ref{app:elbo_stationary:analytic_elbo}) and without mini-batching gradients.}. For each model, we compared the final ELBO to the maximum-likelihood of pPCA finding them to be essentially indistinguishable (as predicted by Lemma~\ref{lemma:vae_global_ppca} and Theorem~\ref{thm:no_elbo_maxima}). For the second set of experiments, we took the pPCA MLE solution for $\bW$ for each number of hidden dimensions and computed the likelihood under the observation noise which maximizes likelihood for 50 hidden dimensions. We observed that adding additional principal components (after 50) will initially improve likelihood but eventually adding more components (after 200) actually decreases the likelihood. In other words, the collapsed solution is actually preferred if the observation noise is not set correctly --- we observe this theoretically through the stability of the stationary points (e.g. Figure~\ref{fig:ppca_stationary}).

\begin{figure}[H]
    \centering
    \includegraphics[width=.8\linewidth]{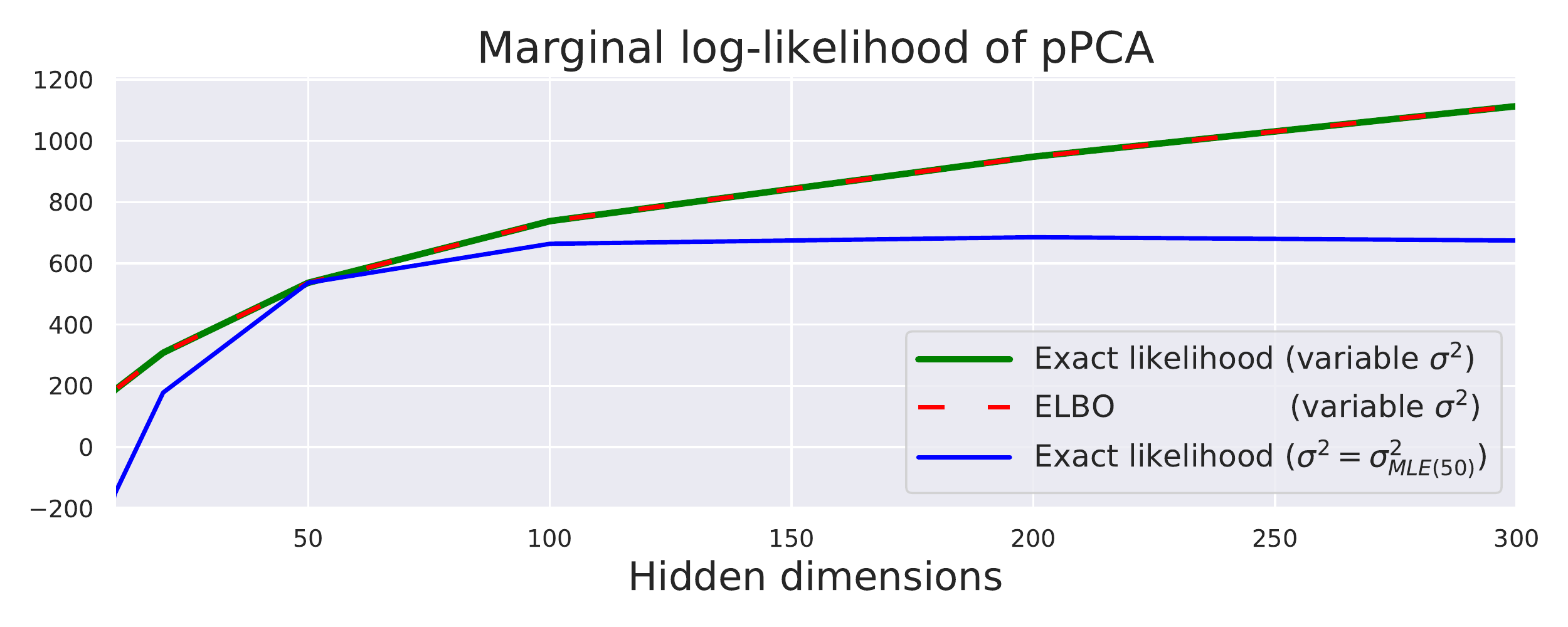}
    \caption{\small The log marginal likelihood and optimal ELBO of MNIST pPCA solutions over increasing hidden dimension. Green represents the MLE solution (global maximum), the red dashed line is the optimal ELBO solution which matches the global optimum. The blue line shows the log marginal likelihood of the solutions using the full decoder weights when $\sigma^2$ is fixed to its MLE solution for 50 hidden dimensions.}
    \label{fig:linear_vae_optimal}
\end{figure}

\begin{minipage}{0.49\linewidth}
\begin{figure}[H]
    \centering
    \includegraphics[width=\linewidth]{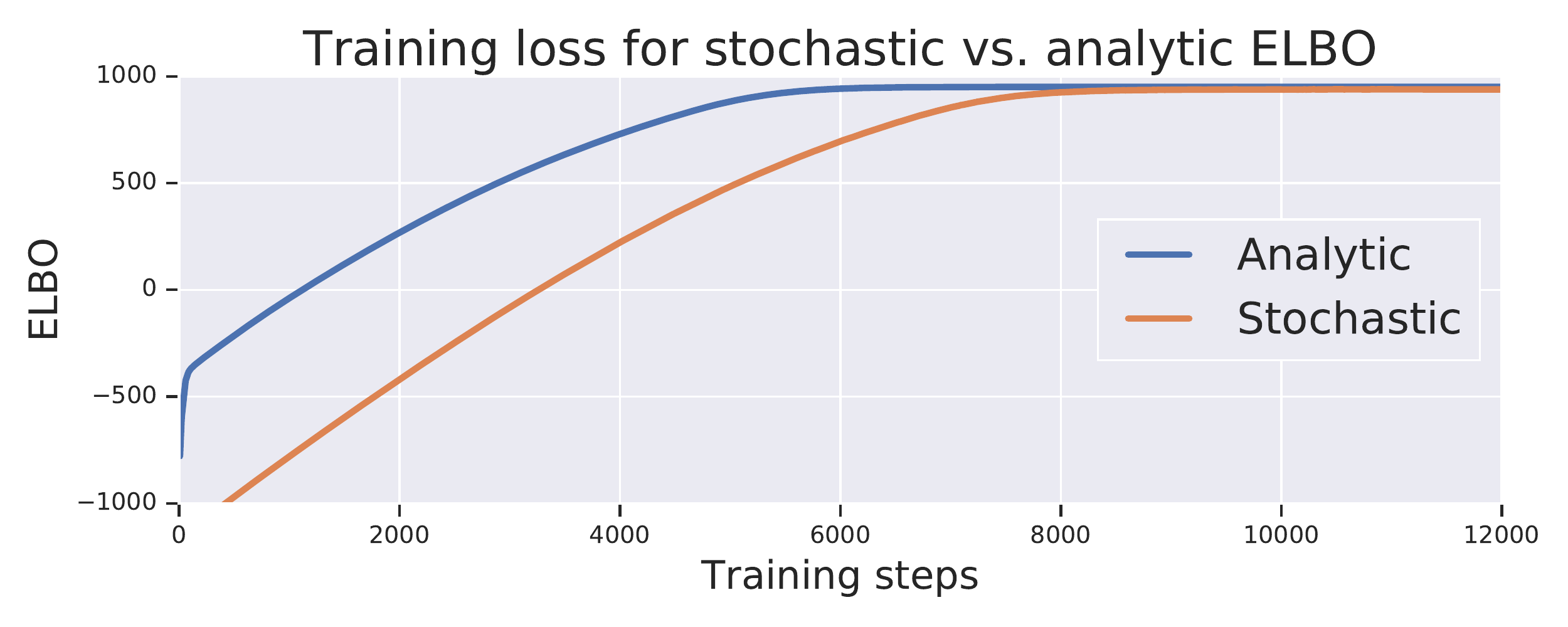}
    \captionof{figure}{\small Stochastic vs analytic ELBO training: using the analytic gradient of the ELBO led to faster convergence and better final ELBO (950.7 vs. 939.3).} 
    \label{fig:linear_vae_stochastic_analytic}
    \end{figure}
    \end{minipage}\hfill%
    \begin{minipage}{0.49\linewidth}
    \begin{figure}[H]
    \centering
    \includegraphics[width=\linewidth]{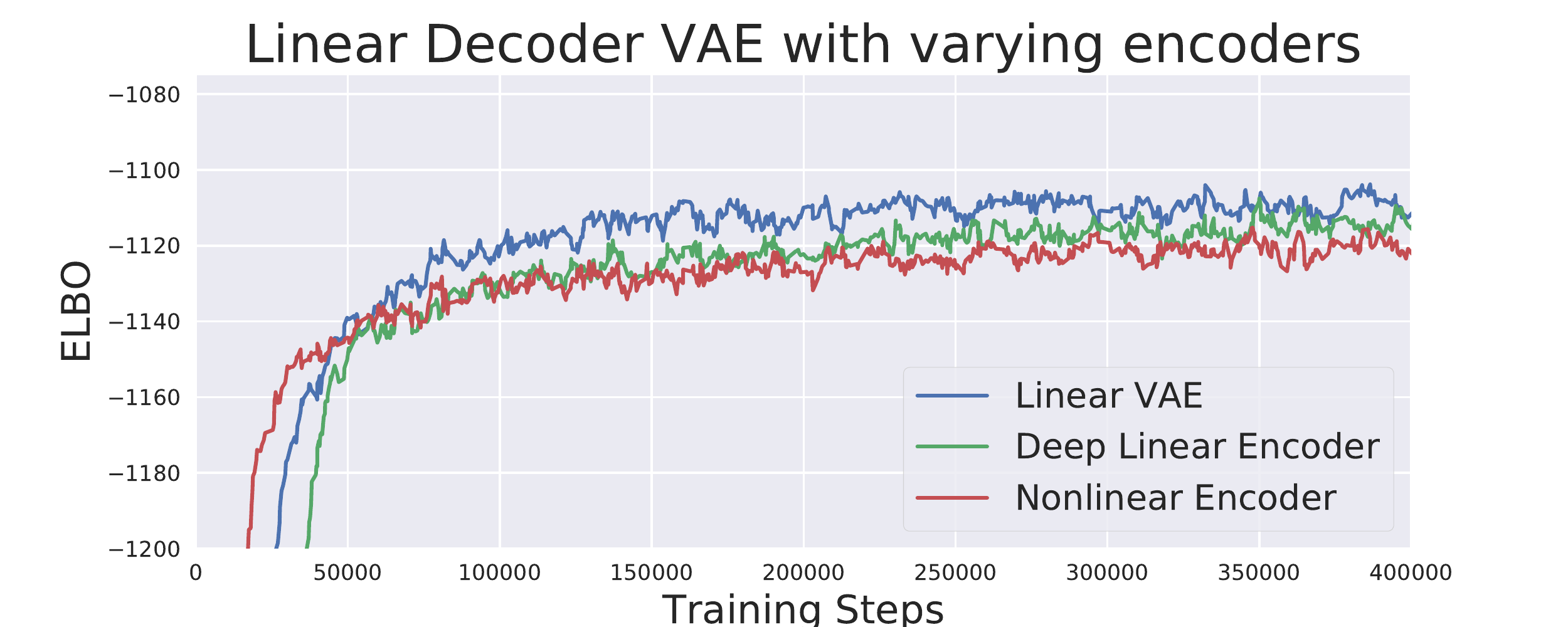}
    \captionof{figure}{\small VAEs with linear decoders trained on real-valued MNIST with nonlinear preprocessing \citep{NIPS2017_6828}. Final average ELBO on training set are (ordered by legend): -1098.2, -1108.7, -1112.1, -1119.6.
    } 
    \label{fig:linear_decoder_only}
    \end{figure}
\end{minipage}

\paragraph{Effect of stochastic ELBO estimates} In general, we are unable to compute the ELBO in closed form and so instead rely on unbiased Monte Carlo estimates using the reparameterization trick. These estimates add high-variance noise and can make optimization more challenging \citep{kingma2013auto}. In the linear model, we can compare the solutions obtained using the stochastic ELBO gradients versus the analytic ELBO\footnote{We use 1000 MNIST images, as before, to enable full-batch training so that the only source of noise is from the reparameterization trick \citep{kingma2013auto}} (Figure~\ref{fig:linear_vae_stochastic_analytic}). Additional experimental details are in Appendix~\ref{app:experiment_details}. We found that stochastic optimization had slower convergence (when compared to analytic training with the same learning rate) and, unsurprisingly, reached a worse final training ELBO value (in other words, worse steady-state risk due to the gradient variance).

\paragraph{Nonlinear Encoders} With a linear decoder and nonlinear encoder, Lemma 1 still holds, and the optimal variational distribution is the same as the true posterior has not changed. However, Corollary 1 and Theorem 1 no longer hold in general. Even a deep linear encoder will not have a unique global maximum and new stationary points (possibly maxima) may be introduced to ELBO in general. To investigate how deeper networks may impact optimization of the probabilistic model, we trained linear decoders with varying encoders using ELBO. We do not expect the linear encoder to be outperformed and indeed the empirical results support this (Figure~\ref{fig:linear_decoder_only}).

%% file: experiments/mnist_vaes.tex
We explored how the analysis of the linear VAEs extends to deep nonlinear models. To do so, we trained VAEs with Gaussian observation models on the MNIST \citep{lecun1998mnist} and CelebA \citep{liu2015faceattributes} datasets. We apply uniform dequantization as in \citet{NIPS2017_6828} in each case. We also adopt the nonlinear logit preprocessing transformation from \citet{NIPS2017_6828} to provide fair comparisons with existing work. We also report results of models trained directly in pixel space in the appendix (there is no significant difference for the hypotheses we test).

\paragraph{Measuring posterior collapse} In order to measure the extent of posterior collapse, we introduce the following definition. We say that latent dimension dimension $i$ has $(\epsilon, \delta)$-collapsed if $\bbP_{\bx \sim p}[KL(q(z_i|\bx)||p(z_i)) < \epsilon] \geq 1-\delta$. Note that the linear VAE can suffer $(0,0)$-collapse. To estimate this practically, we compute the proportion of data samples which induce a variational distribution with KL divergence less than $\epsilon$ and finally report the percentage of dimensions which have $(\epsilon,\delta)$-collapsed. Throughout this work, we fix $\delta = 0.01$ and vary $\epsilon$.

\paragraph{Investigating $\sigma^2$}

We trained MNIST VAEs with 2 hidden layers in both the decoder and encoder, ReLU activations, and 200 latent dimensions.
We first evaluated training with fixed values of the observation noise, $\sigma^2$.
This mirrors many public VAE implementations where $\sigma^2$ is fixed to 1 throughout training (also observed by \citet{dai2018diagnosing}), however, our linear analysis suggests that this is suboptimal. Then, we consider the setting where the observation noise and VAE weights are learned simultaneously.

In Table~\ref{tab:summarized_eval_table} we report the final ELBO of nonlinear VAEs trained on real-valued MNIST. For fixed $\sigma^2$, we found that the final models could have significant differences in ELBO which were maintained even after tuning $\sigma^2$ to the learned representations --- the converged representations are less good when $\sigma^2$ is too large as predicted by the linear model. Additionally, we report the final ELBO values when the model is trained while learning $\sigma^2$ with different initial values of $\sigma^2$. The gap in performance across different initializations is smaller than for fixed $\sigma^2$ but is still significant. The linear VAE does not predict this gap which suggests that learning $\sigma^2$ correctly is more challenging in the nonlinear case.

\begin{table}[t]
\resizebox{\textwidth}{!}{%
\begin{tabular}{|c|d{2.3}|d{1.3}|l|@{\hspace*{.1cm}}|l|r|@{\hspace*{.1cm}}|r|r|  }
    \hline
    & \multicolumn{2}{|c|}{Model} & \multirow{2}{*}{ELBO} & \multirow{2}{*}{$\sigma^2$-tuned ELBO} & \multirow{2}{*}{Tuned $\sigma^2$} & \mc{Posterior} & \mc{KL} \\ \cline{2-3}
    & \mc{Init $\sigma^2$} & \mc{Final $\sigma^2$} & & & & \mc{collapse (\%)}  & \mc{Divergence} \\
     \hline
     \multirow{10}{*}{\rotatebox[origin=c]{90}{MNIST}}
& \multicolumn{2}{c|}{10.0} & $-1450.3 \pm 4.2$ & $-1098.2 \pm 28.3 $ & 1.797 & 89.88 & $28.8 \pm 1.4$\\
& \multicolumn{2}{c|}{1.0} & $-1022.1 \pm 5.4$ & $-1018.3 \pm 5.3 $ & 1.145 & 27.38 & $125.4 \pm 4.2$\\
& \multicolumn{2}{c|}{0.1} & $-3697.3 \pm 493.3$ & $-1190.8 \pm 37.4 $ & 0.968 & 3.25 & $368.7 \pm 94.6$\\
& \multicolumn{2}{c|}{0.01} & $-38612.5 \pm 1189.8$ & $-2090.8 \pm 975.1 $ & 0.877 & 0.00 & $695.9 \pm 118.1$\\
& \multicolumn{2}{c|}{0.001} & $-504259.1 \pm 49149.8$ & $-1744.7 \pm 48.4 $ & 0.810 & 0.00 & $756.2 \pm 12.6$\\
    \cline{2-8}
& 10.0 & 1.320 & $-1022.2 \pm 4.5$ & $-1022.3 \pm 4.6 $ & 1.318 & 73.75 & $73.8 \pm 9.8$\\
& 1.0 & 1.183 & $-1011.1 \pm 2.7$ & $-1011.1 \pm 2.8 $ & 1.182 & 47.88 & $106.3 \pm 2.5$\\
& 0.1 & 1.194 & $-1025.4 \pm 8.6$ & $-1025.4 \pm 8.6 $ & 1.195 & 29.25 & $116.1 \pm 11.4$\\
& 0.01 & 1.194 & $-1030.6 \pm 3.5$ & $-1030.5 \pm 3.5 $ & 1.191 & 23.00 & $121.9 \pm 7.7$\\
& 0.001 & 1.208 & $-1038.7 \pm 5.6$ & $-1038.8 \pm 5.6 $ & 1.209 & 27.00 & $124.9 \pm 1.6$\\
    \hline\hline
     \multirow{10}{*}{\rotatebox[origin=c]{90}{CELEBA 64}}
     & \multicolumn{2}{|c|}{10.0} & $-73328.4 \pm 0.49$ & $-55186.7 \pm 35.1$ & 0.2040 & 80.56 & $56.12 \pm 0.4$ \\
     & \multicolumn{2}{|c|}{1.0} & $-59841.8 \pm 30.1$ & $-51294.8 \pm 333.7$ & 0.1020 & 2.52 & $213.4 \pm 6.3$ \\
     & \multicolumn{2}{|c|}{0.1} & $-50760.3 \pm 353.4$ & $-50698.5 \pm 393.9$ & 0.0883 & 32.72 & $483.8 \pm 36.2$ \\
     & \multicolumn{2}{|c|}{0.01} & $-82478.7 \pm 1823.3$ & $-51373.9 \pm 213.3$ & 0.0817 & 0.00 & $1624.2 \pm 8.8$ \\
     & \multicolumn{2}{|c|}{0.001} & $-531924.5 \pm 17177.6$ & $-57381.5 \pm 512.6$ & 0.0296 & 0.00 & $2680.2 \pm 41.5$ \\
    \cline{2-8}
    & 10.0 & 0.0962 & $-51109.5 \pm 408.2$ & $-51109.5 \pm 408.3$ & 0.0963 & 53.32 & $364.5 \pm 26.4$ \\
    & 1.0 & 0.0875 & $-50631.2 \pm 163.4$ & $-50631.0 \pm 163.3$ & 0.0875 & 54.76 & $462.2 \pm 20.0$ \\
    & 0.1 & 0.0863 & $-50646.9 \pm 269.0$ & $-50645.9 \pm 267.5$ & 0.0869 & 28.84 & $520.9 \pm 11.7$ \\
    & 0.01 & 0.0911 & $-51285.0 \pm 708.1$ & $-51284.8 \pm 708.1$ & 0.0963 & 5.64 & $557.0 \pm 50.5$ \\
    & 0.001 & 0.1040 & $-51695.1 \pm 322.4$ & $-51694.8 \pm 322.7$ & 0.0974 & 0.00 & $537.5 \pm 46.2$ \\
    \hline\hline
\end{tabular}}
\caption{\footnotesize Evaluation of deep Gaussian VAEs (averaged over 5 trials) on real-valued MNIST. We report the ELBO on the training set in all cases. Collapse percent gives the percentage of latent dimensions which are within 0.01 KL of the prior for at least 99\% of the encoder inputs.}%
\label{tab:summarized_eval_table}
\end{table}

\begin{figure}
    \centering
    \includegraphics[width=0.98\linewidth]{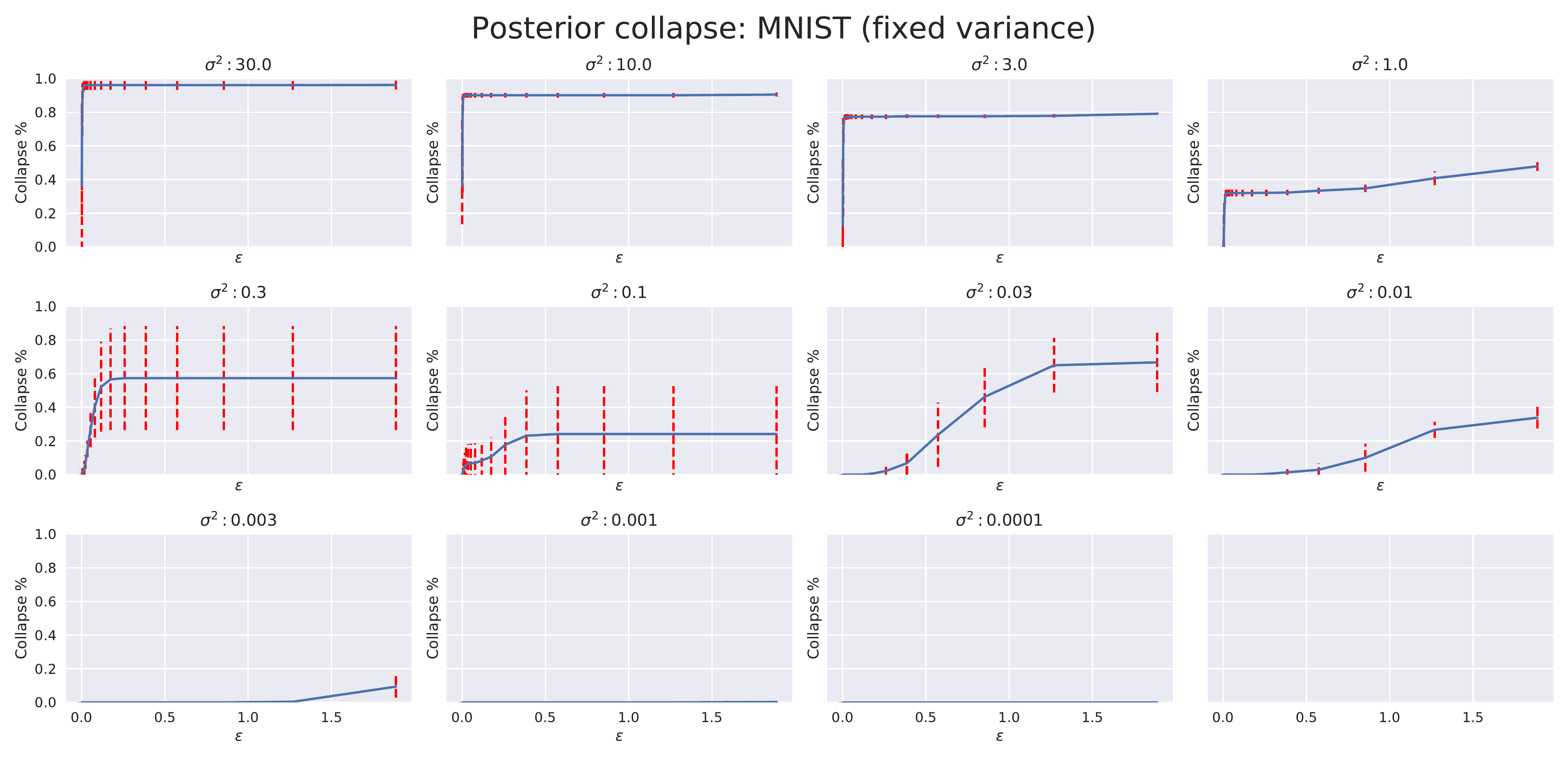}
    \caption{Posterior collapse percentage as a function of $\epsilon$-threshold for a deep VAE trained on MNIST. We measure posterior collapse for trained networks as the proportion of latent dimensions that are within $\epsilon$ KL divergence of the prior for at least a $1-\delta$ proportion of the training data points ($\delta=0.01$ in the plots).}
    \label{fig:posterior_collapse_thresholds_fixed}
\end{figure}

\begin{figure}
    \centering
    \includegraphics[width=0.98\linewidth]{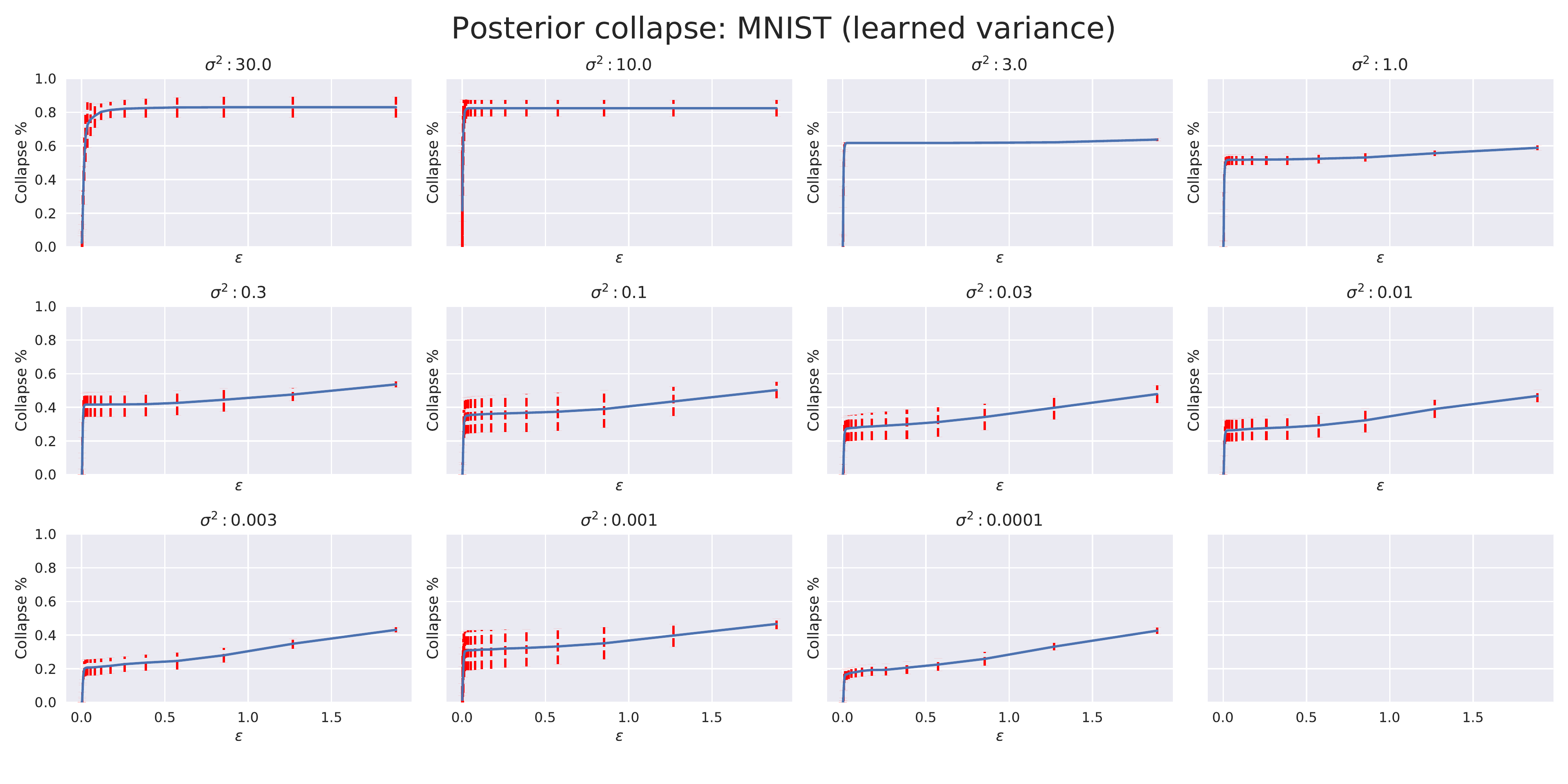}
    \caption{Posterior collapse percentage as a function of $\epsilon$-threshold for a deep VAE trained on MNIST. We measure posterior collapse for trained networks as the proportion of latent dimensions that are within $\epsilon$ KL divergence of the prior for at least a $1-\delta$ proportion of the training data points ($\delta=0.01$ in the plots).}
    \label{fig:posterior_collapse_thresholds_learned}
\end{figure}

Despite the large volume of work studying posterior collapse it has not been measured in a consistent way (or even defined so). In Figure~\ref{fig:posterior_collapse_thresholds_fixed} and Figure~\ref{fig:posterior_collapse_thresholds_learned} we measure posterior collapse for trained networks as described above (we chose $\delta=0.01$). By considering a range of $\epsilon$ values we found this was (moderately) robust to stochasticity in data preprocessing. We observed that for large choices of $\sigma^2$ initialization the variational distribution matches the prior closely. This was true even when $\sigma^2$ is learned --- suggesting that local optima may contribute to posterior collapse in deep VAEs.

%% file: experiments/deeper_vaes.tex
\paragraph{CelebA VAEs} We trained deep convolutional VAEs with 500 hidden dimensions on images from the CelebA dataset (resized to 64x64). We trained the CelebA VAEs with different fixed values of $\sigma^2$ and compared the ELBO before and after tuning $\sigma^2$ to the learned representations (Table~\ref{tab:summarized_eval_table}). Further, we explored training the CelebA VAE while learning $\sigma^2$ over varied initializations of the observation noise. The VAE is sensitive to the initialization of the observation noise even when $\sigma^2$ is learned (in particular, in terms of the number of collapsed dimensions).

%% file: conclusion.tex
\section{Discussion}\label{sec:discussion}
\vspace{-.2cm}

By analyzing the correspondence between linear VAEs and pPCA, this paper makes
significant progress towards understanding the causes of posterior collapse.
We show that for simple linear VAEs posterior collapse is caused by ill-conditioning
of the stationary points in the log marginal likelihood objective.
We demonstrate empirically that the same optimization issues play a role in deep non-linear VAEs.
Finally, we find that linear VAEs are useful theoretical test-cases for evaluating existing hypotheses on VAEs
and we encourage researchers to consider studying their hypotheses in the linear VAE setting.

%% file: acknowledgements.tex
\section{Acknowledgements}

This work was guided by many conversations with and feedback from our colleagues. In particular, we thank Durk Kingma, Alex Alemi, and Guodong Zhang for invaluable feedback on early versions of this work.

%% file: appendix/appendices.tex
\begin{appendix}

\input{appendix/ppca_stationary.tex}

\input{appendix/vae_identifiability.tex}
\input{appendix/elbo_stationary.tex}

\input{appendix/bernoulli_ppca.tex}
\input{appendix/extended_related.tex}
\input{appendix/experiment_details.tex}

\end{appendix}

%% file: appendix/ppca_stationary.tex
\section{Stationary points of pPCA}\label{app:ppca_stationary}

Here we briefly summarize the analysis of \citep{tipping1999probabilistic} with some simple additional observations. We recommend that interested readers study Appendix~A of \citet{tipping1999probabilistic} for the full details.  We begin by formulating the conditions for stationary points of $\sum_{\bx_i}\log p(\bx_i)$:

\begin{equation}
\bS \bC^{-1}\bW = \bW
\end{equation}

Where $\bS$ denotes the sample covariance matrix (assuming we set $\bmu = \bmu_{MLE}$, \emph{which we do throughout}), and $\bC = \bW\bW^T + \sigma^2 I$ (note that the dimensionality is different to $\mathbf{M}$). There are three possible solutions to this equation, (1) $\bW = \mathbf{0}$, (2) $\bC = \bS$, or (3) the more general solutions. (1) and (2) are not particularly interesting to us, so we focus herein on (3).

We can write $\bW = \bU\bL\bV^T$ using its singular value decomposition. Substituting back into the stationary points equation, we recover the following:

\begin{equation}
\bS\bU\bL = \bU(\sigma^2I + \bL^2)\bL
\end{equation}

Noting that $\bL$ is diagonal, if the $j^{th}$ singular value ($l_j$) is non-zero, this gives $\bS \bu_j = (\sigma^2 + l^2_j)\bu_j$, where $u_j$ is the $j^{th}$ column of $\bU$. Thus, $\bu_j$ is an eigenvector of $\bS$ with eigenvalue $\lambda_j = \sigma^2 + l^2_j$. For $l_j=0$, $\bu_j$ is arbitrary.

Thus, all potential solutions can be written as, $\bW = U_q(K_q-\sigma^2I)^{1/2}\bR$, with singular values written as $k_j = \sigma^2$ or $\sigma^2 + l^2_j$ and with $\bR$ representing an arbitrary orthogonal matrix.

From this formulation, one can show that the global optimum is attained with $\sigma^2 = \sigma^2_{MLE}$ and $U_q$ and $K_q$ chosen to match the leading singular vectors and values of $\bS$.

\subsection{Stability of stationary point solutions}

Consider stationary points of the form, $\bW = \bU_q(K_q - \sigma^2I)^{1/2}$ where $\bU_q$ contains arbitrary eigenvectors of $\bS$. In the original pPCA paper they show that all solutions except the leading principal components correspond to saddle points in the optimization landscape. However, this analysis depends critically on $\sigma^2$ being set to the true maximum likelihood estimate. Here we repeat their analysis, considering other (fixed) values of $\sigma^2$.

We consider a small perturbation to a column of $\bW$, of the form $\epsilon \bu_j$ . To analyze the stability of the perturbed solution, we check the sign of the dot-product of the perturbation with the likelihood gradient at $\mathbf{w}_i + \epsilon \bu_j$. Ignoring terms in $\epsilon^2$ we can write the dot-product as,

\begin{equation}
\epsilon N (\lambda_j / k_i - 1) \bu_j^T \bC^{-1} \bu _j
\end{equation}

Now, $\bC^{-1}$ is positive definite and so the sign depends only on $\lambda_j / k_i - 1$. The stationary point is stable (local maxima) only if the sign is negative.  If $k_i = \lambda_i$ then the maxima is stable only when $\lambda_i > \lambda_j$, in words, the top $q$ principal components are stable. However, we must also consider the case $k=\sigma^2$. \citet{tipping1999probabilistic} show that if $\sigma^2 = \sigma^2_{MLE}$, then this also corresponds to a saddle point as $\sigma^2$ is the average of the smallest eigenvalues meaning some perturbation will be unstable (except in a special case which is handled separately).

However, what happens if $\sigma^2$ is not set to be the maximum likelihood estimate? In this case, it is possible that there are no unstable perturbation directions (that is, $\lambda_j < \sigma^2$ for too many $j$). In this case when $\sigma^2$ is fixed, there are local optima where $\bW$ has zero-columns --- the same solutions that we observe in non-linear VAEs corresponding to posterior collapse. Note that when $\sigma^2$ is learned in non-degenerate cases the local maxima presented above become saddle points where $\sigma^2$ is made smaller by its gradient. In practice, we find that even when $\sigma^2$ is learned in the non-linear case local maxima exist.

%% file: appendix/vae_identifiability.tex
\section{Identifiability of the linear VAE}\label{app:vae_identifiability}

Linear autoencoders suffer from a lack of identifiability which causes the decoder columns to span the principal component subspace instead of recovering it. \citet{kunin2019loss} showed that adding regularization to the linear autoencoder improves the identifiability --- forcing the columns to be identified up to an arbitrary orthogonal transformation, as in pPCA. Here we show that linear VAEs are able to fully identify the principal components.

We once again consider the linear VAE from \eqref{eqn:linear_vae}:

\begin{align*}
    \begin{split}
        p(\bx \mid \bz) ~=~ \cN( \bW \bz + \bmu, \sigma^2 \bI), \\
        q(\bz \mid \bx) ~=~ \cN(\bV (\bx - \bmu), \bD),
    \end{split}
\end{align*}

The output of the VAE, $\tilde{\bx}$ is distributed as,

\[ \tilde{\bx}|\bx \sim \cN(\bW\bV(\bx-\bmu) + \bmu, \bW\bD\bW^T). \]

Therefore, the output of the linear VAE is invariant to the following transformation:

\begin{align}
    \begin{split}
        \bW &\leftarrow \bW \mathbf{A}, \\
        \bV &\leftarrow \mathbf{A}^{-1}\bV, \\
        \bD &\leftarrow \mathbf{A}^{-1} \bD \mathbf{A}^{-1},
    \end{split}
\end{align}

where $\mathbf{A}$ is a diagonal matrix with non-zero entries so that $\bD$ is well-defined. However, this transformation changes the variational distribution which affects the loss through the KL term. As argued in Corollary 1, this means that the global optimum is unique for ELBO up to ordering of the eigenvalues/eigenvectors.

At the global optimum, the ordering can be recovered by computing the squared Euclidean norm of the columns of $\bW$ (which correspond to the singular values) and ordering according to these quantities. In other words, $\bR$ is a permutation matrix which can be computed exactly.

%% file: appendix/elbo_stationary.tex
\section{Stationary points of ELBO}\label{app:elbo_stationary}

Here we present details on the analysis of the stationary points of the ELBO objective. To begin, we first derive closed-form solutions to the components of the log marginal likelihood (including the ELBO). The VAE we focus on is the one presented in \eqref{eqn:linear_vae}, with a linear encoder, linear decoder, Gaussian prior, and Gaussian observation model.

\subsection{Analytic ELBO of the Linear VAE}\label{app:elbo_stationary:analytic_elbo}

Remember that one can express the log marginal likelihood as:
\begin{equation}\label{eqn:marginal_loglik}
    \log p(\bx) = \overset{(A)}{KL(q(\bz|\bx)||p(\bz|\bx))} - \overset{(B)}{KL(q(\bz|\bx)||p(\bz))} + \overset{(C)}{\expect_{q(\bz|\bx)}\left[ \log p(\bx|\bz) \right]}.
\end{equation}
Each of the terms (A-C) can be expressed in closed form for the linear VAE. Note that the KL term (A) is minimized when the variational distribution is exactly the true posterior distribution. This is possible when the columns of the decoder are orthogonal.

The term (B) can be expressed as,
\begin{equation}
    KL(q(\bz|\bx)||p(z)) = 0.5(-\log\det \bD + (\bx - \bmu)^T\bV^T\bV(\bx - \bmu) + tr(\bD) - q).
\end{equation}

The term (C) can be expressed as,
\begin{align}
\expect_{q(\bz|\bx)}\left[ \log p(\bx|\bz) \right] &= \expect_{q(\bz|\bx)}\left[ -(\bW\bz - (\bx - \bmu))^T(\bW\bz - (\bx - \bmu))/2\sigma^2 - \frac{d}{2}\log 2\pi\sigma^2 \right ] \\
&= \expect_{q(\bz|\bx)}\left[ \frac{-(\bW\bz)^T(\bW\bz) + 2(\bx - \bmu)^T\bW\bz - (\bx - \bmu)^T(\bx - \bmu)}{2\sigma^2} - \frac{d}{2}\log 2\pi\sigma^2 \right ].
\end{align}
Noting that $\bW\bz \sim \cN\left(\bW\bV(\bx-\bmu), \bW\bD\bW^T\right)$, we can compute the expectation analytically and obtain,
\begin{align}
\expect_{q(\bz|\bx)}\left[ \log p(\bx|\bz) \right] &= \frac{1}{2\sigma^2} [ -tr(\bW\bD\bW^T) - (\bx - \bmu)^T\bV^T\bW^T\bW\bV(\bx-\bmu) \\
&+ 2(\bx - \bmu)^T\bW\bV(\bx-\bmu) - (\bx - \bmu)^T(\bx - \bmu)] - \frac{d}{2}\log 2\pi\sigma^2.
\end{align}

\subsection{Finding stationary points}
To compute the stationary points we must take derivatives with respect to $\bmu, \bD, \bW, \bV, \sigma^2$. As before, we have $\bmu=\bmu_{MLE}$ at the global maximum and for simplicity we fix $\bmu$ here for the remainder of the analysis.

Taking the marginal likelihood over the whole dataset, at the stationary points we have,

\begin{align}
\frac{\partial}{\partial \bD} (-(B) + (C)) &= \frac{N}{2}(\bD^{-1} - \bI - \frac{1}{\sigma^2}\text{diag}(\bW^T\bW)) = 0 \label{eqn:D_stationary} \\
\frac{\partial}{\partial \bV} (-(B) + (C)) &= \frac{N}{\sigma^2}(\bW^T - (\bW^T\bW + \sigma^2\bI)\bV)\bS = 0 \label{eqn:V_stationary} \\ 
\frac{\partial}{\partial \bW} (-(B) + (C)) &= \frac{N}{\sigma^2}(\bS\bV^T - \bD\bW-\bW\bV\bS\bV^T) = 0 \label{eqn:W_stationary}
\end{align}

The above are computed using standard matrix derivative identities \citep{petersen2008matrix}. These equations yield the expected solution for the variational distribution directly. From \eqref{eqn:D_stationary} we compute $\bD^* = \sigma^2(\text{diag}(\bW^T\bW) + \sigma^2 \bI)^{-1}$ and $\bV^* = \bM^{-1}\bW^T$, recovering the true posterior mean in all cases and getting the correct posterior covariance when the columns of $\bW$ are orthogonal. We will now proceed with the proof of Theorem~\ref{thm:no_elbo_maxima}.

\elbomaxima*

\begin{proof}
If the columns of $\bW$ are orthogonal then the log marginal likelihood is recovered exactly at all stationary points. This is a direct consequence of the posterior mean \emph{and} covariance being recovered exactly at all stationary points so that (1) is zero.

We must give separate treatment to the case where there is a stationary point without orthogonal columns of $\bW$. Suppose we have such a stationary point, using the singular value decomposition we can write $\bW = \bU\bL\bR^T$, where $\bU$ and $\bR$ are orthogonal matrices. Note that $\log p(\bx)$ is invariant to the choice of $\bR$ \citep{tipping1999probabilistic}. However, the choice of $\bR$ does affect the first term (1) of \eqref{eqn:marginal_loglik}: this term is minimized when $\bR = \bI$, and thus the ELBO must increase.

To formalize this argument, we compute (1) at a stationary point. From above, at every stationary point the mean of the variational distribution exactly matches the true posterior. Thus the KL simplifies to:

\begin{align}
KL(q(\bz|\bx)||p(\bz|\bx)) &= \frac{1}{2}\left( tr(\frac{1}{\sigma^2} \bM \bD) -q +q\log\sigma^2-\log(\det\bM\det\bD) \right), \\
&= \frac{1}{2}\left( tr(\bM \widetilde{\bM}^{-1}) -q -\log\frac{\det\bM}{\det\widetilde{\bM}} \right), \\
&= \frac{1}{2}\left( \sum_{i=1}^{q}\frac{\bM_{ii}}{\bM_{ii}} - q - \log\det\bM + \log\det\widetilde{\bM}
\right), \\
&= \frac{1}{2}\left( \log\det\widetilde{\bM} - \log\det\bM\right), \\
\end{align}

where $\widetilde{\bM} = \text{diag}(\bW^T\bW) + \sigma^2 \bI$. Now consider applying a small rotation to $\bW$: $\bW \mapsto \bW\bR_\epsilon$. As the optimal $\bD$ and $\bV$ are continuous functions of $\bW$, this corresponds to a small perturbation of these parameters too for a sufficiently small rotation. Importantly, $\log\det\bM$ remains fixed for any orthogonal choice of $\bR_\epsilon$ but $\log\det\widetilde{\bM}$ does not. Thus, we choose $\bR_\epsilon$ to minimize this term. In this manner, (1) shrinks meaning that the ELBO (-2)+(3) must increase. Thus if the stationary point existed, it must have been a saddle point.

We now describe how to construct such a small rotation matrix. First note that without loss of generality we can assume that $\det(\bR) = 1$. (Otherwise, we can flip the sign of a column of $\bR$ and the corresponding column of $\bU$.) And additionally, we have $\bW\bR=\bU\bL$, which is orthogonal.

The Special Orthogonal group of determinant 1 orthogonal matrices is a compact, connected Lie group and therefore the exponential map from its Lie algebra is surjective. This means that we can find an upper-triangular matrix $\mathbf{B}$, such that $\bR = \exp \{ \mathbf{B} - \mathbf{B}^T\}$. Consider $\bR_\epsilon = \exp \{ \frac{1}{n(\epsilon)}(\mathbf{B} - \mathbf{B}^T)\}$, where $n(\epsilon)$ is an integer chosen to ensure that the elements of $\mathbf{B}$ are within $\epsilon > 0$ of zero. This matrix is a rotation in the direction of $\bR$ which we can make arbitrarily close to the identity by a suitable choice of $\epsilon$. This is verified through the Taylor series expansion of $\bR_\epsilon = I + \frac{1}{n(\epsilon)}(\mathbf{B} - \mathbf{B}^T) + O(\epsilon^2)$. Thus, we have identified a small perturbation to $\bW$ (and $\bD$ and $\bV$) which decreases the posterior KL (A) but keeps the log marginal likelihood constant. Thus, the ELBO increases and the stationary point must be a saddle point.
\end{proof}

%% file: appendix/bernoulli_ppca.tex
\subsection{Bernoulli Probabilistic PCA}\label{app:bernoulli_ppca}

We would like to extend our linear analysis to the case where we have a Bernoulli observation model, as this setting also suffers severely from posterior collapse. The analysis may also shed light on more general categorical observation models which have also been used. Typically, in these settings a continuous latent space is still used (for example, \citet{bowman2015generating}).

We will consider the following model,

\begin{align}
    \begin{split}
        p(\bz) = \cN(0, \bI),\\
        p(\bx|\bz) = \text{Bernoulli}(\by), \\
        \by = \sigma(\bW\bz + \bmu)
    \end{split}
\end{align}

where $\sigma$ denotes the sigmoid function, $\sigma(y) = 1 / (1 + \exp(-y))$ and we assume an independent Bernoulli observation model over $\bx$.

Unfortunately, under this model it is difficult to reason about the stationary points. There is no closed form solution for the marginal likelihood $p(\bx)$ or the posterior distribution $p(\bz|\bx)$. Numerical integration methods exist which may make it easy to evaluate this quantity in practice but they will not immediately provide us a good gradient signal.

We can compute the density function for $\by$ using the change of variables formula. Noting that $\bW\bz + \bmu \sim \cN(\bmu, \bW\bW^T)$, we recover the following logit-Normal distribution:

\begin{equation}
    f(\by) = \frac{1}{\sqrt{2\pi|\bW\bW^T|}}\frac{1}{\Pi_i y_i(1-y_i)}\exp \{-\frac{1}{2}\left(\log(\frac{\by}{1-\by}) - \bmu\right)^T(\bW\bW^T)^{-1}\left(\log(\frac{\by}{1-\by}) - \bmu\right)\}
\end{equation}

We can write the marginal likelihood as,

\begin{align}
    p(\bx) &= \int p(\bx|\bz)p(\bz) d\bz, \\
    &= \expect_\bz \left[ \by(\bz)^\bx (1 - \by(\bz))^{1-\bx} \right],
\end{align}

where $(\cdot)^\bx$ is taken to be elementwise. Unfortunately, the expectation of a logit-normal distribution has no closed form \citep{atchison1980logistic} and so we cannot tractably compute the marginal likelihood.

Similarly, under ELBO we need to compute the expected reconstruction error. This can be written as,

\begin{equation}
\expect_{q(\bz|\bx)}[\log p(\bx|\bz)] = \int \by(\bz)^\bx (1 - \by(\bz))^{1-\bx} \cN(\bz; \bV(\bx - \bmu), \bD) d\bz,
\end{equation}

another intractable integral.

%% file: appendix/extended_related.tex
\section{Related Work (Extended)}\label{app:related}
Due to the large volume of work studying posterior collapse in variational autoencoders, we have included here an extended discussion of related work. We utilize this additional space to provide a more in-depth discussion of the related work presented in the main paper and to highlight additional work.

\citet{tomczak2017vae} introduce the VampPrior, a hierarchical learned prior for VAEs. \citet{tomczak2017vae} show empirically that such a learned prior can mitigate posterior collapse (which they refer to as inactive stochastic units). While the authors provide limited theoretical support for the efficacy of their method in reducing posterior collapse, they claim intuitively that by enabling multi-modal prior distributions the KL term is less likely to force inactive units --- possibly by reducing the impact of local optima corresponding to posterior collapse.

In the main paper we discuss the work of \citet{dai2017hidden}, which connect robust PCA methods and VAEs. In particular, Section~2 of their manuscript studies the case of a linear decoder and shows that, when the encoder takes the form of the optimal variational distribution, the ELBO of the resulting VAE collapses into the pPCA objective. We study the ELBO without optimality assumptions on the linear encoder and characterize the optimization landscape with no additional assumptions. They claim further that all minima of the (encoder-optimal) ELBO objective are globally optimal --- we show in fact that for a linear encoder there is a fully identifiable global optimum.

\citet{dai2018diagnosing} discuss the important of the observation noise, and in fact show that under some assumptions the optimal observation noise should shrink to zero (Theorem~4 in their work). These assumptions amount to the number of latent dimensions exceeding the dimensionality of the true data manifold. However, in the linear model (whose latent dimensions do not exceed the input space dimensionality) the optimal variance does not shrink towards zero and is instead given by the sum of the variance lost in the linear projection. Note that this does not violate the results of \citet{dai2018diagnosing}, but highlights the need to consider model capacity against data complexity, as in \citet{alemi2017fixing}.

%% file: appendix/experiment_details.tex
\section{Experiment details}\label{app:experiment_details}

We used Tensorflow \citep{tensorflow2015-whitepaper} for our experiments with linear and deep VAEs. In each case, the models were trained using a single GPU.

\paragraph{Visualizing stationary points of pPCA} For this experiment we computed the pPCA MLE using a subset of 1000 random training images from the MNIST dataset. We evaluate and plot the log marginal likelihood in closed form on this same subset. In this case, we did not dequantize or apply any nonlinear processing to the data.

\paragraph{Stochastic vs. Analytic VAE} We trained linear VAEs with 200 hidden dimensions. We used full-batch training with 1000 MNIST digits samples randomly from the training set (the same data as used to produce Figure~\ref{fig:linear_vae_optimal}). We trained each model with the Adam optimizer and a fixed learning rate, grid searching to find the learning rate which gave the best ELBO after 12000 training steps in the range $\{ 0.0001, 0.0003, 0.001, 0.003\}$. For both models, 0.001 provided the best final ELBO.

\paragraph{MNIST VAE} The VAEs we trained on MNIST all had the same architecture: 784-1024-512-k-512-1024-784.  The Gaussian likelihood is fairly uncommon for this dataset, which is nearly binary, but it provides a good setting for us to investigate our theoretical findings. To dequantize the data, we added uniform random noise and rescaled the pixel values to be in the range $[0,1]$. We then applied a nonlinear logistic transform as in \citep{NIPS2017_6828}. The VAE parameters were optimized jointly using the Adam optimizer \citep{kingma2014adam}. We trained the VAE for 1000 epochs total, keeping the learning rate fixed throughout. We performed a grid search over learning rates in the range $\{0.0001, 0.0003, 0.001, 0.003 \}$ and reported results for the model which achieved the best training ELBO.

\paragraph{CelebA VAE} We used the convolutional architecture proposed by \citet{higgins2016beta} trained on 64x64 images from the CelebA dataset \citep{liu2015faceattributes}. Otherwise, the experimental procedure followed that of the MNIST VAEs with the nonlinear preprocessing hyperparameters set as in \citep{NIPS2017_6828}.

\input{appendix/additional_results.tex}

%% file: appendix/additional_results.tex
\subsection{Additional results}
\subsubsection{Evaluating KL Annealing}

We found that KL-annealing may provide temporary relief from posterior collapse but that if $\sigma^2$ is not learned simultaneously then the collapsed solution is recovered. In Figure~\ref{fig:fixed_sigma_beta_anneal_kl} we show the proportion of units collapsed by threshold for several fixed choices of $\sigma^2$ when $\beta$ is annealed from 0 to 1 over the first 100 epochs. The solid lines correspond to the final model while the dashed line corresponds to the model at 80 epochs of training. KL-annealing was able to reduce posterior collapse initially but eventually fell back to the collapsed solution.

\begin{figure}[t]
 \includegraphics[width=0.6\linewidth]{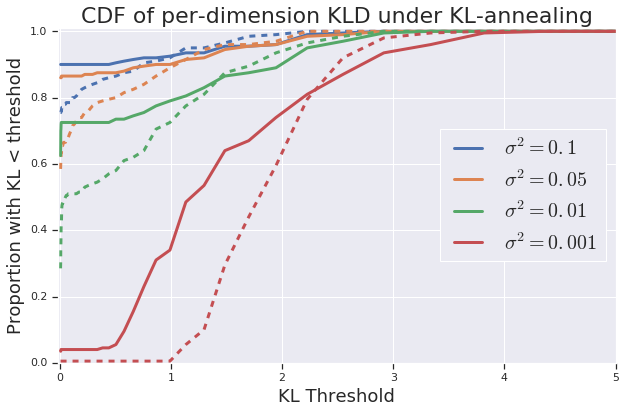}%
 \caption{\small Proportion of inactive units thresholded by KL divergence when using 0-1 KL-annealing and a fixed value of $\sigma^2$. The solid line represents the final model while the dashed line is the model after only 80 epochs of training. KL annealing reduces posterior collapse during the early stages of training but ultimately fails to escape these sub-optimal solutions as the KL weight is increased.
 }
\label{fig:fixed_sigma_beta_anneal_kl}%
\end{figure}

\begin{figure}[t]
    \centering
    \includegraphics[width=\linewidth]{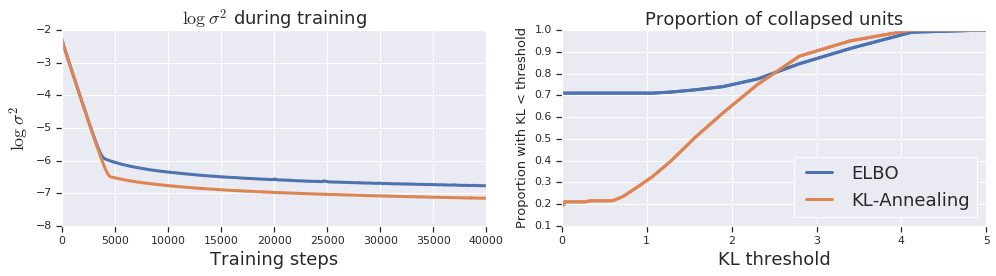}
    \caption{\small Comparing learned solutions using KL-Annealing versus standard ELBO training when $\sigma^2$ is learned.}
    \label{fig:learn_sigma_annealing}
\end{figure}

After finding that KL-annealing alone was insufficient to prevent posterior collapse we explored KL annealing while learning $\sigma^2$. Based on our analysis in the linear case we expect that this should work well: while $\beta$ is small the model should be able to learn to reduce $\sigma^2$. We trained using the same KL schedule and also with standard ELBO while learning $\sigma^2$. The results are presented in Figure~\ref{fig:learn_sigma_annealing} and Figure~\ref{fig:celeba_learn_sigma}. Under the ELBO objective, $\sigma^2$ is reduced somewhat but ultimately a large degree of posterior collapse is present. Using KL-annealing, the VAE is able to learn a much smaller $\sigma^2$ value and ultimately reduces posterior collapse. This suggests that the non-linear VAE dynamics may be similar to the linear case when suitably conditioned.

\begin{figure}[t]
    \centering
    \includegraphics[width=0.9\linewidth]{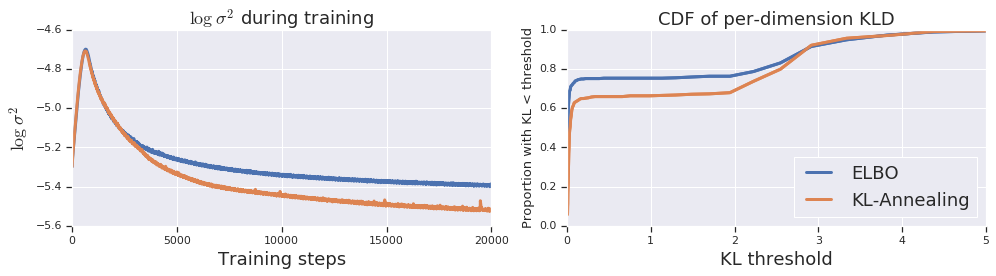}
    \caption{\small Learning $\sigma^2$ for CelebA VAEs with standard ELBO training and KL-Annealing. KL-Annealing enables a smaller $\sigma^2$ to be learned and reduces posterior collapse.} 
    \label{fig:celeba_learn_sigma}
\end{figure}

\subsubsection{Full results tables}

\input{experiments/big_table.tex}

\begin{figure}
    \centering
    \includegraphics[width=0.98\linewidth]{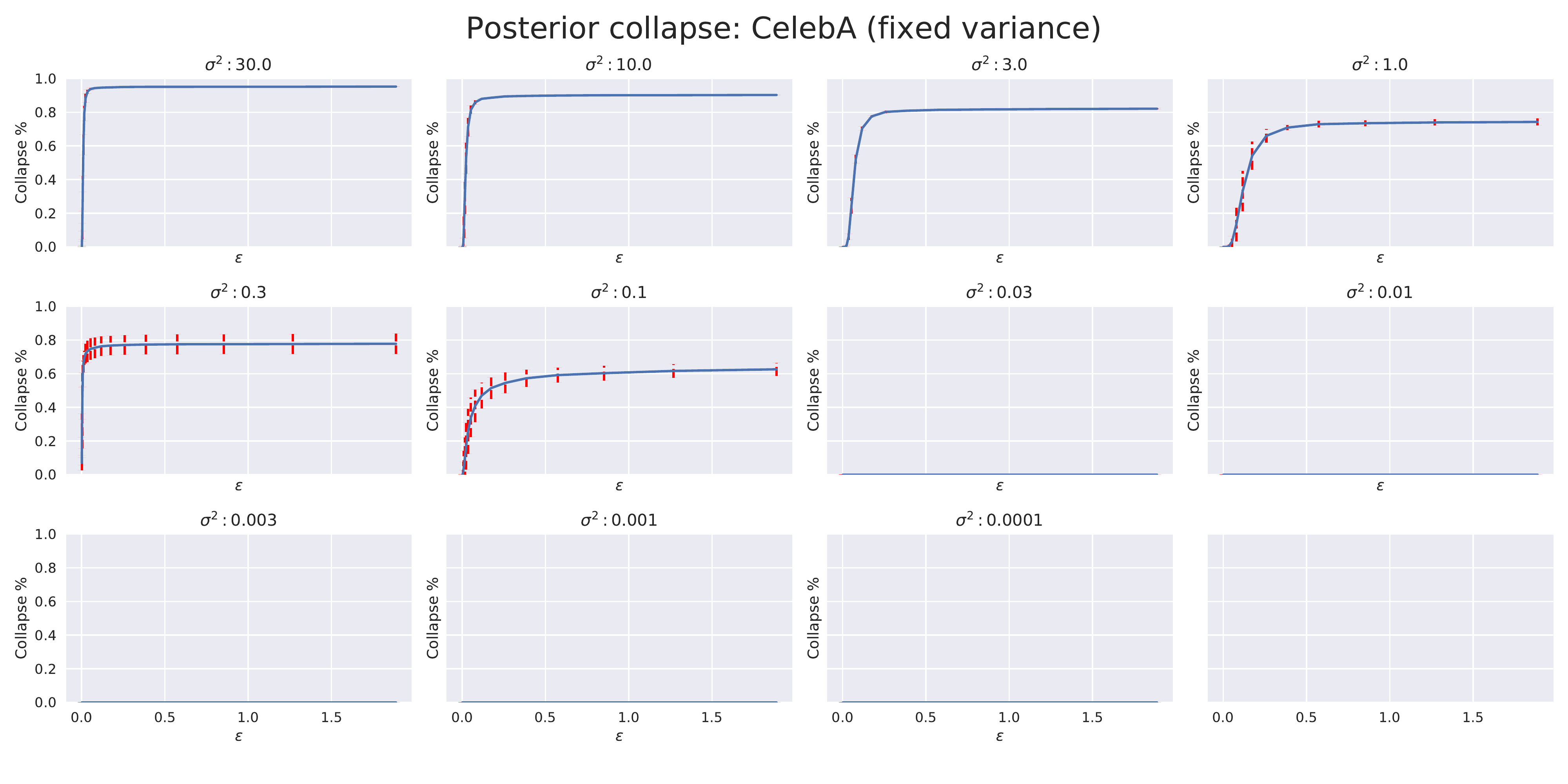}
    \caption{Posterior collapse percentage as a function of $\epsilon$-threshold for a deep VAE trained on CelebA with fixed $\sigma^2$. We measure posterior collapse for trained networks as the proportion of latent dimensions that are within $\epsilon$ KL divergence of the prior for at least a $1-\delta$ proportion of the training data points ($\delta=0.01$ in the plots).}
    \label{fig:posterior_collapse_thresholds_fixed_celeba}
\end{figure}

\begin{figure}
    \centering
    \includegraphics[width=0.98\linewidth]{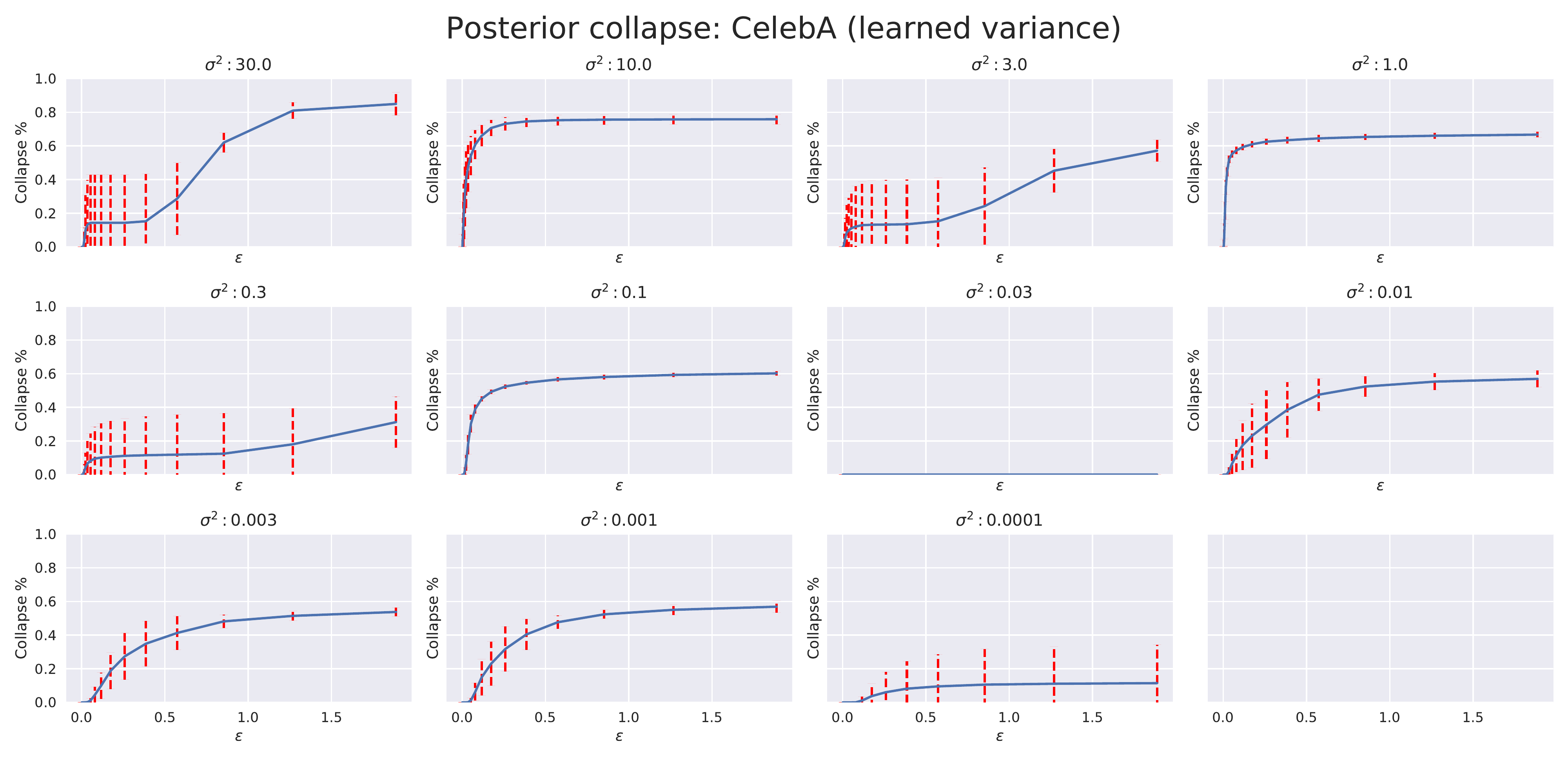}
    \caption{Posterior collapse percentage as a function of $\epsilon$-threshold for a deep VAE trained on CelebA with learned $\sigma^2$. We measure posterior collapse for trained networks as the proportion of latent dimensions that are within $\epsilon$ KL divergence of the prior for at least a $1-\delta$ proportion of the training data points ($\delta=0.01$ in the plots).}
    \label{fig:posterior_collapse_thresholds_learned_celeba}
\end{figure}

\subsubsection{Qualitative Results}
Reconstructions from the KL-Annealed CelebA model are shown in Figure~\ref{fig:celeba_recons}. We also show the output of interpolating in the latent space in Figure~\ref{fig:celeba_interpolate}. To produce the latter plot, we compute the variational mean of 3 input points (top left, top right, bottom left) and interpolate linearly on the plane between them. We also extrapolate out to a fourth point (bottom right), which lies on the plane defined by the other points.

\begin{figure}
    \centering
    \includegraphics[width=0.9\linewidth]{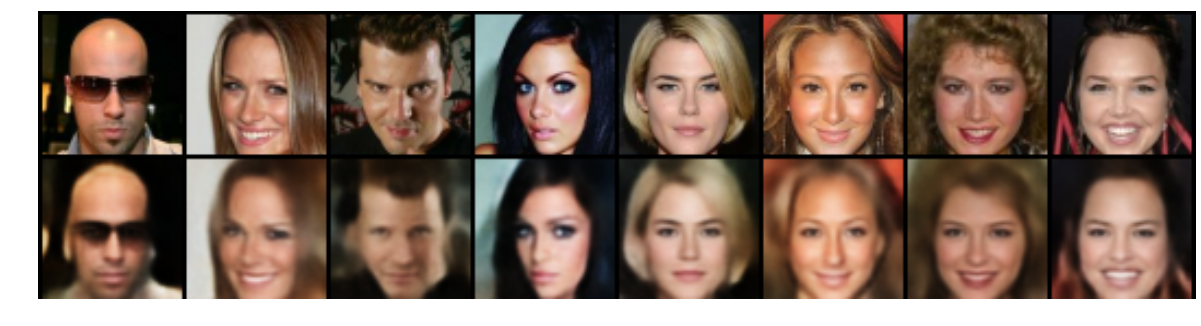}
    \caption{Reconstructions from the convolutional VAE trained with KL-Annealing on CelebA.} 
    \label{fig:celeba_recons}
\end{figure}

\begin{figure}
    \centering
    \includegraphics[width=0.9\linewidth]{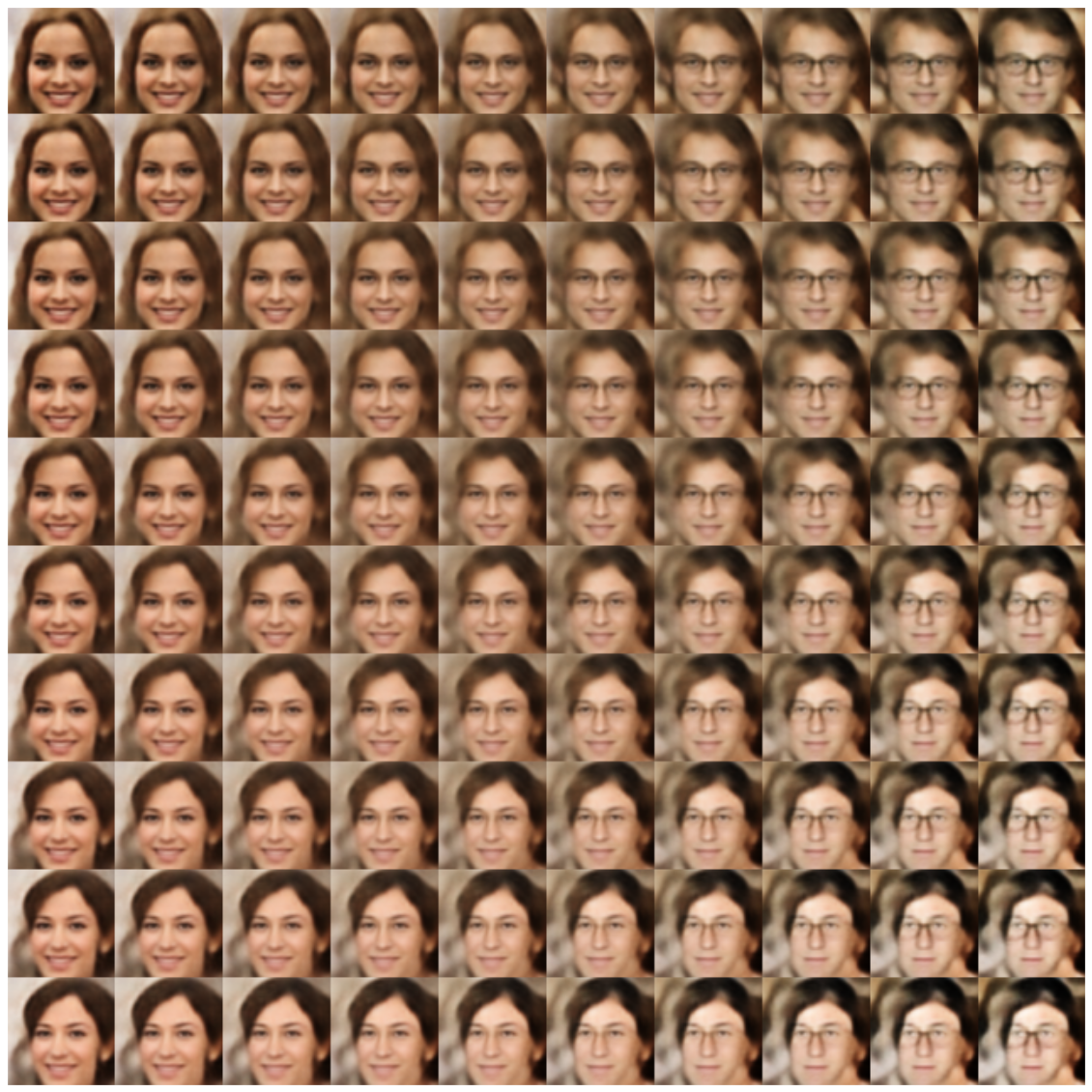}
    \caption{Latent space interpolations from the convolutional VAE trained with KL-Annealing on CelebA.} 
    \label{fig:celeba_interpolate}
\end{figure}

%% file: experiments/big_table.tex
\begin{table}[t]
\resizebox{\textwidth}{!}{%
\begin{tabular}{|c|c|c|l|@{\hspace*{.1cm}}|l|r|@{\hspace*{.1cm}}|r|l|  }
    \hline
    & \multicolumn{2}{|c|}{Model} & \multirow{2}{*}{ELBO} & \multirow{2}{*}{$\sigma^2$-tuned ELBO} & \multirow{2}{*}{Tuned $\sigma^2$} & Posterior & KL \\ \cline{2-3}
    & Init $\sigma^2$ & Final $\sigma^2$ & & & & collapse (\%)  & Divergence \\
     \hline
     \multirow{20}{*}{\rotatebox[origin=c]{90}{MNIST}} & \multicolumn{2}{|c|}{30.0} & $-1850.4 \pm 29.0$ & $-1374.9 \pm 199.0 $ & 4.451 & 95.00 & $10.9 \pm 6.7$\\
& \multicolumn{2}{|c|}{10.0} & $-1450.3 \pm 4.2$ & $-1098.2 \pm 28.3 $ & 1.797 & 89.88 & $28.8 \pm 1.4$\\
& \multicolumn{2}{|c|}{3.0} & $-1114.9 \pm 1.1$ & $-1018.8 \pm 1.0 $ & 1.361 & 76.75 & $58.5 \pm 1.4$\\
& \multicolumn{2}{|c|}{1.0} & $-1022.1 \pm 5.4$ & $-1018.3 \pm 5.3 $ & 1.145 & 27.38 & $125.4 \pm 4.2$\\
& \multicolumn{2}{|c|}{0.3} & $-1816.7 \pm 270.6$ & $-1104.6 \pm 6.2 $ & 1.275 & 2.00 & $179.3 \pm 85.9$\\
& \multicolumn{2}{|c|}{0.1} & $-3697.3 \pm 493.3$ & $-1190.8 \pm 37.4 $ & 0.968 & 3.25 & $368.7 \pm 94.6$\\
& \multicolumn{2}{|c|}{0.03} & $-18549.3 \pm 4892.0$ & $-1283.2 \pm 63.3 $ & 1.470 & 0.00 & $305.3 \pm 75.4$\\
& \multicolumn{2}{|c|}{0.01} & $-38612.5 \pm 1189.8$ & $-1403.1 \pm 21.0 $ & 1.006 & 0.00 & $560.9 \pm 32.4$\\
& \multicolumn{2}{|c|}{0.003} & $-139538.8 \pm 21148.5$ & $-2090.8 \pm 975.1 $ & 0.877 & 0.00 & $695.9 \pm 118.1$\\
& \multicolumn{2}{|c|}{0.001} & $-504259.1 \pm 49149.8$ & $-1744.7 \pm 48.4 $ & 0.810 & 0.00 & $756.2 \pm 12.6$\\
    \cline{2-8}
    & 30.0 & 1.478 & $-1060.9 \pm 23.1$ & $-1061.0 \pm 23.0 $ & 1.476 & 33.75 & $70.9 \pm 13.8$\\
& 10.0 & 1.32 & $-1022.2 \pm 4.5$ & $-1022.3 \pm 4.6 $ & 1.318 & 73.75 & $73.8 \pm 9.8$\\
& 3.0 & 1.178 & $-1004.6 \pm 1.4$ & $-1004.5 \pm 1.3 $ & 1.181 & 58.38 & $99.8 \pm 1.5$\\
& 1.0 & 1.183 & $-1011.1 \pm 2.7$ & $-1011.1 \pm 2.8 $ & 1.182 & 47.88 & $106.3 \pm 2.5$\\
& 0.3 & 1.195 & $-1020.0 \pm 6.0$ & $-1019.9 \pm 6.1 $ & 1.191 & 37.75 & $111.6 \pm 6.1$\\
& 0.1 & 1.194 & $-1025.4 \pm 8.6$ & $-1025.4 \pm 8.6 $ & 1.195 & 29.25 & $116.1 \pm 11.4$\\
& 0.03 & 1.197 & $-1030.6 \pm 6.6$ & $-1030.5 \pm 6.6 $ & 1.198 & 22.62 & $120.2 \pm 10.5$\\
& 0.01 & 1.194 & $-1030.6 \pm 3.5$ & $-1030.5 \pm 3.5 $ & 1.191 & 23.00 & $121.9 \pm 7.7$\\
& 0.003 & 1.19 & $-1033.7 \pm 2.3$ & $-1033.6 \pm 2.3 $ & 1.187 & 16.62 & $126.4 \pm 6.8$\\
& 0.001 & 1.208 & $-1038.7 \pm 5.6$ & $-1038.8 \pm 5.6 $ & 1.209 & 27.00 & $124.9 \pm 1.6$\\
    \hline\hline
\end{tabular}}
\caption{\footnotesize Full evaluation of deep Gaussian VAEs (averaged over 5 trials) on real-valued MNIST with nonlinear preprocessing \citep{NIPS2017_6828}. Collapse percent gives the percentage of latent dimensions which are within 0.01 KL of the prior for at least 99\% of the encoder inputs.}%
\label{tab:full_logit_mnist_eval}
\end{table}

\begin{table}[t]
\resizebox{\textwidth}{!}{%
\begin{tabular}{|c|c|c|l|@{\hspace*{.1cm}}|l|r|@{\hspace*{.1cm}}|r|l|  }
    \hline
    & \multicolumn{2}{|c|}{Model} & \multirow{2}{*}{ELBO} & \multirow{2}{*}{$\sigma^2$-tuned ELBO} & \multirow{2}{*}{Tuned $\sigma^2$} & Posterior & KL \\ \cline{2-3}
    & Init $\sigma^2$ & Final $\sigma^2$ & & & & collapse (\%)  & Divergence \\
     \hline
     \multirow{20}{*}{\rotatebox[origin=c]{90}{CELEBA 64}} & \multicolumn{2}{|c|}{30.0} & $-79986.2 \pm 0.10$ & $-57883.8 \pm 19.3$ & 0.423 & 93.68 & $26.0 \pm 0.2$ \\
     & \multicolumn{2}{|c|}{10.0} & $-73328.4 \pm 0.49$ & $-55186.7 \pm 35.1$ & 0.204 & 80.56 & $56.12 \pm 0.4$ \\
     & \multicolumn{2}{|c|}{3.0} & $-66145.6 \pm 2.44$ & $-52828.5 \pm 58.6$  & 0.132 & 20.64 & $120.4 \pm 1.4$ \\
     & \multicolumn{2}{|c|}{1.0} & $-59841.8 \pm 30.1$ & $-51294.8 \pm 333.7$ & 0.102 & 2.52 & $213.4 \pm 6.3$ \\
     & \multicolumn{2}{|c|}{0.3} & $-54370.4 \pm 849.9$ & $-52155.2 \pm 1855.2$ & 0.122 & 74.52 & $267.2 \pm 51.9$ \\
     & \multicolumn{2}{|c|}{0.1} & $-50760.3 \pm 353.4$ & $-50698.5 \pm 393.9$ & 0.0883 & 32.72 & $483.8 \pm 36.2$ \\
     & \multicolumn{2}{|c|}{0.03} & $-64322.8 \pm 312.9$ & $-58077.9 \pm 206.2$ & 0.0463 & 0.00 & $1521.1 \pm 11.6$  \\
     & \multicolumn{2}{|c|}{0.01} & $-82478.7 \pm 1823.3$ & $-51373.9 \pm 213.3$ & 0.0817 & 0.00 & $1624.2 \pm 8.8$ \\
     & \multicolumn{2}{|c|}{0.003} & $-192967.7 \pm 4410.4$ & $-51978.4 \pm 159.3$ & 0.0685 & 0.00 & $2108.4 \pm 26.2$ \\
     & \multicolumn{2}{|c|}{0.001} & $-531924.5 \pm 17177.6$ & $-57381.5 \pm 512.6$ & 0.0296 & 0.00 & $2680.2 \pm 41.5$ \\
    \cline{2-8}
    & 30.0 & 0.478 & $-57773.0 \pm 3622.9$ & $-56068.5 \pm 2771.0$ & 0.475 & 14.20 & $221.7 \pm 99.0$ \\
    & 10.0 & 0.0962 & $-51109.5 \pm 408.2$ & $-51109.5 \pm 408.3$ & $0.0963$ & 53.32 & $364.5 \pm 26.4$ \\
    & 3.0 & 0.0891 & $-50813.2 \pm 229.7$ & $-50813.3 \pm 229.7$ & 0.0889 & 10.96 & $545.2 \pm 5.5$  \\
    & 1.0 & 0.0875 & $-50631.2 \pm 163.4$ & $-50631.0 \pm 163.3$ & 0.0875 & 54.76 & $462.2 \pm 20.0$ \\
    & 0.3 & 0.0890 & $-50963.4 \pm 331.2$ & $-50963.2 \pm 331.3$ & 0.0892 & 7.96 & $670.7 \pm 79.2$ \\
    & 0.1 & 0.0863 & $-50646.9 \pm 269.0$ & $-50645.9 \pm 267.5$ & 0.0869 & 28.84 & $520.9 \pm 11.7$ \\
    & 0.03 & 0.121 & $-53263.4 \pm 71.5$ & $-53263.3 \pm 71.3$ & 0.126 & 0.00 & $856.2 \pm 19.7$ \\
    & 0.01 & 0.0911 & $-51285.0 \pm 708.1$ & $-51284.8 \pm 708.1$ & 0.0963 & 5.64 & $557.0 \pm 50.5$ \\
    & 0.003 & 0.0952 & $-51056.4 \pm 1216.9$ & $-51055.9 \pm 1217.4$ & 0.094 & 0.80 & $577.4 \pm 30.4$ \\
    & 0.001 & 0.104 & $-51695.1 \pm 322.4$ & $-51694.8 \pm 322.7$ & 0.0974 & 0.00 & $537.5 \pm 46.2$ \\
    \hline\hline
\end{tabular}}
\caption{\footnotesize Full evaluation of deep Gaussian VAEs (averaged over 5 trials) on real-valued CelebA with nonlinear preprocessing \citep{NIPS2017_6828}. Collapse percent gives the percentage of latent dimensions which are within 0.01 KL of the prior for at least 99\% of the encoder inputs.}%
\label{tab:full_logit_celeba_eval}
\end{table}

\begin{table}[t]
\resizebox{\textwidth}{!}{%
\begin{tabular}{|c|c|c|l|@{\hspace*{.1cm}}|l|r|@{\hspace*{.1cm}}|r|l|  }
    \hline
    & \multicolumn{2}{|c|}{Model} & \multirow{2}{*}{ELBO} & \multirow{2}{*}{$\sigma^2$-tuned ELBO} & \multirow{2}{*}{Tuned $\sigma^2$} & Posterior & KL \\ \cline{2-3}
    & Init $\sigma^2$ & Final $\sigma^2$ & & & & collapse (\%)  & Divergence \\
     \hline
     \multirow{20}{*}{\rotatebox[origin=c]{90}{MNIST}} & \multicolumn{2}{|c|}{30.0} & $-6402.0 \pm 0.0$ & $-6248.4 \pm 197.2 $ & 22.323 & 0.00 & $0.0 \pm 0.0$\\
& \multicolumn{2}{|c|}{10.0} & $-5973.1 \pm 0.0$ & $-5821.0 \pm 194.6 $ & 7.443 & 0.00 & $0.0 \pm 0.0$\\
& \multicolumn{2}{|c|}{3.0} & $-5507.1 \pm 0.1$ & $-5360.4 \pm 185.4 $ & 2.235 & 1.70 & $0.6 \pm 0.3$\\
& \multicolumn{2}{|c|}{1.0} & $-5087.9 \pm 3.1$ & $-4954.7 \pm 156.9 $ & 0.747 & 0.00 & $4.5 \pm 2.3$\\
& \multicolumn{2}{|c|}{0.3} & $-4638.4 \pm 3.6$ & $-4516.8 \pm 137.9 $ & 0.225 & 0.00 & $12.5 \pm 1.5$\\
& \multicolumn{2}{|c|}{0.1} & $-4243.1 \pm 17.6$ & $-4154.6 \pm 62.1 $ & 0.076 & 0.00 & $25.6 \pm 3.0$\\
& \multicolumn{2}{|c|}{0.03} & $-3820.7 \pm 13.9$ & $-3785.2 \pm 26.6 $ & 0.027 & 0.00 & $55.8 \pm 2.1$\\
& \multicolumn{2}{|c|}{0.01} & $-3508.4 \pm 12.3$ & $-3483.5 \pm 13.1 $ & 0.009 & 0.00 & $112.8 \pm 6.7$\\
& \multicolumn{2}{|c|}{0.003} & $-3267.3 \pm 2.6$ & $-3247.1 \pm 2.8 $ & 0.003 & 0.00 & $252.2 \pm 2.1$\\
& \multicolumn{2}{|c|}{0.001} & $-3137.7 \pm 5.2$ & $-3136.7 \pm 5.4 $ & 0.001 & 0.00 & $422.7 \pm 2.6$\\
    \cline{2-8}
    & 30.0 & 0.067 & $-4398.7 \pm 0.0$ & $-4398.7 \pm 0.0 $ & 0.067 & 0.00 & $0.0 \pm 0.0$\\
& 10.0 & 0.044 & $-4146.3 \pm 309.2$ & $-4146.3 \pm 309.2 $ & 0.044 & 0.00 & $30.1 \pm 36.9$\\
& 3.0 & 0.01 & $-3736.3 \pm 14.3$ & $-3736.4 \pm 14.3 $ & 0.010 & 0.00 & $73.7 \pm 1.9$\\
& 1.0 & 0.008 & $-3673.0 \pm 17.7$ & $-3672.9 \pm 17.7 $ & 0.008 & 0.00 & $85.2 \pm 2.5$\\
& 0.3 & 0.006 & $-3569.8 \pm 26.4$ & $-3569.8 \pm 26.4 $ & 0.006 & 0.00 & $100.8 \pm 3.7$\\
& 0.1 & 0.003 & $-3355.8 \pm 7.6$ & $-3355.8 \pm 7.6 $ & 0.003 & 0.00 & $151.7 \pm 2.4$\\
& 0.03 & 0.001 & $-3138.9 \pm 10.6$ & $-3139.0 \pm 10.6 $ & 0.001 & 0.00 & $275.4 \pm 3.1$\\
& 0.01 & 0.001 & $-3126.1 \pm 5.0$ & $-3126.1 \pm 5.0 $ & 0.001 & 0.00 & $349.3 \pm 5.4$\\
& 0.003 & 0.001 & $-3161.4 \pm 4.0$ & $-3161.3 \pm 4.0 $ & 0.001 & 0.00 & $373.5 \pm 7.5$\\
& 0.001 & 0.001 & $-3145.4 \pm 6.1$ & $-3145.4 \pm 6.1 $ & 0.001 & 0.00 & $378.4 \pm 7.7$\\
    \hline\hline
\end{tabular}}
\caption{\footnotesize Evaluation of deep Gaussian VAEs (averaged over 5 trials) on real-valued MNIST without any nonlinear preprocessing. Collapse percent gives the percentage of latent dimensions which are within 0.01 KL of the prior for at least 99\% of the encoder inputs.}%
\label{tab:full_pixel_mnist_eval}
\end{table}

\begin{table}[t]
\resizebox{\textwidth}{!}{%
\begin{tabular}{|c|c|c|l|@{\hspace*{.1cm}}|l|r|@{\hspace*{.1cm}}|r|l|  }
    \hline
    & \multicolumn{2}{|c|}{Model} & \multirow{2}{*}{ELBO} & \multirow{2}{*}{$\sigma^2$-tuned ELBO} & \multirow{2}{*}{Tuned $\sigma^2$} & Posterior & KL \\ \cline{2-3}
    & Init $\sigma^2$ & Final $\sigma^2$ & & & & collapse (\%)  & Divergence \\
     \hline
     \multirow{20}{*}{\rotatebox[origin=c]{90}{CELEBA 64}} & \multicolumn{2}{|c|}{30.0} & $-79986.2 \pm 0.10$ & $-57883.8 \pm 19.3$ & 0.423 & 93.68 & $26.0 \pm 0.19$ \\
     & \multicolumn{2}{|c|}{10.0} & $-73328.4 \pm 0.49$ & $-55186.7 \pm 35.1$ & 0.204 & 80.56 & $56.12 \pm 0.42$ \\
     & \multicolumn{2}{|c|}{3.0} & $-66145.6 \pm 2.44$ & $-52828.5 \pm 58.6$  & 0.132 & 20.64 & $120.4 \pm 1.37$ \\
     & \multicolumn{2}{|c|}{1.0} & $-59841.8 \pm 30.1$ & $-51294.8 \pm 333.7$ & 0.102 & 2.52 & $213.4 \pm 6.3$ \\
     & \multicolumn{2}{|c|}{0.3} & $-54370.4 \pm 849.9$ & $-52155.2 \pm 1855.2$ & 0.122 & 74.52 & $267.2 \pm 51.9$ \\
     & \multicolumn{2}{|c|}{0.1} & $-50760.3 \pm 353.4$ & $-50698.5 \pm 393.9$ & 0.0883 & 32.72 & $483.8 \pm 36.2$ \\
     & \multicolumn{2}{|c|}{0.03} & $-64322.8 \pm 312.9$ & $-58077.9 \pm 206.2$ & 0.0463 & 0.00 & $1521.1 \pm 11.6$  \\
     & \multicolumn{2}{|c|}{0.01} & $-82478.7 \pm 1823.3$ & $-51373.9 \pm 213.3$ & 0.0817 & 0.00 & $1624.2 \pm 8.78$ \\
     & \multicolumn{2}{|c|}{0.003} & $-192967.7 \pm 4410.4$ & $-51978.4 \pm 159.3$ & 0.0685 & 0.00 & $2108.4 \pm 26.2$ \\
     & \multicolumn{2}{|c|}{0.001} & $-531924.5 \pm 17177.6$ & $-57381.5 \pm 512.6$ & 0.0296 & 0.00 & $2680.2 \pm 41.45$ \\
    \cline{2-8}
    & 30.0 & 0.005 & $-53179.6 \pm 450.2$ & $-53179.6 \pm 450.3 $ & 0.005 & 0.00 & $302.8 \pm 29.8$\\
& 10.0 & 0.004 & $-51748.5 \pm 178.2$ & $-51748.5 \pm 178.2 $ & 0.004 & 0.00 & $482.3 \pm 24.7$\\
& 3.0 & 0.004 & $-51548.9 \pm 154.1$ & $-51548.9 \pm 154.2 $ & 0.004 & 0.00 & $489.5 \pm 21.8$\\
& 1.0 & 0.004 & $-51356.9 \pm 79.1$ & $-51356.9 \pm 79.1 $ & 0.004 & 0.00 & $516.3 \pm 18.0$\\
& 0.3 & 0.004 & $-51767.7 \pm 369.2$ & $-51767.7 \pm 369.1 $ & 0.004 & 22.00 & $439.7 \pm 33.3$\\
& 0.1 & 0.004 & $-51637.3 \pm 163.3$ & $-51637.1 \pm 163.5 $ & 0.004 & 0.00 & $577.3 \pm 13.5$\\
& 0.03 & 0.004 & $-51792.6 \pm 163.4$ & $-51792.6 \pm 163.6 $ & 0.004 & 45.48 & $484.6 \pm 22.6$\\
& 0.01 & 0.004 & $-51925.1 \pm 99.8$ & $-51924.9 \pm 99.8 $ & 0.004 & 0.00 & $627.8 \pm 20.6$\\
& 0.003 & 0.004 & $-52111.2 \pm 149.0$ & $-52111.0 \pm 148.8 $ & 0.004 & 42.80 & $466.9 \pm 13.9$\\
& 0.001 & 0.004 & $-52060.1 \pm 171.8$ & $-52060.0 \pm 171.9 $ & 0.004 & 0.0 & $645.6 \pm 19.2$\\
    \hline\hline
\end{tabular}}
\caption{\footnotesize Evaluation of deep Gaussian VAEs (averaged over 5 trials) on real-valued CelebA without any nonlinear preprocessing. Collapse percent gives the percentage of latent dimensions which are within 0.01 KL of the prior for at least 99\% of the encoder inputs.}%
\label{tab:full_pixel_celeba_eval}
\end{table}

%% file: main.bbl
\begin{thebibliography}{41}
\providecommand{\natexlab}[1]{#1}
\providecommand{\url}[1]{\texttt{#1}}
\expandafter\ifx\csname urlstyle\endcsname\relax
  \providecommand{\doi}[1]{doi: #1}\else
  \providecommand{\doi}{doi: \begingroup \urlstyle{rm}\Url}\fi

\bibitem[Abadi et~al.(2015)Abadi, Agarwal, Barham, Brevdo, Chen, Citro,
  Corrado, Davis, Dean, Devin, Ghemawat, Goodfellow, Harp, Irving, Isard, Jia,
  Jozefowicz, Kaiser, Kudlur, Levenberg, Man\'{e}, Monga, Moore, Murray, Olah,
  Schuster, Shlens, Steiner, Sutskever, Talwar, Tucker, Vanhoucke, Vasudevan,
  Vi\'{e}gas, Vinyals, Warden, Wattenberg, Wicke, Yu, and
  Zheng]{tensorflow2015-whitepaper}
M.~Abadi, A.~Agarwal, P.~Barham, E.~Brevdo, Z.~Chen, C.~Citro, G.~S. Corrado,
  A.~Davis, J.~Dean, M.~Devin, S.~Ghemawat, I.~Goodfellow, A.~Harp, G.~Irving,
  M.~Isard, Y.~Jia, R.~Jozefowicz, L.~Kaiser, M.~Kudlur, J.~Levenberg,
  D.~Man\'{e}, R.~Monga, S.~Moore, D.~Murray, C.~Olah, M.~Schuster, J.~Shlens,
  B.~Steiner, I.~Sutskever, K.~Talwar, P.~Tucker, V.~Vanhoucke, V.~Vasudevan,
  F.~Vi\'{e}gas, O.~Vinyals, P.~Warden, M.~Wattenberg, M.~Wicke, Y.~Yu, and
  X.~Zheng.
\newblock {TensorFlow}: Large-scale machine learning on heterogeneous systems,
  2015.
\newblock URL \url{https://www.tensorflow.org/}.
\newblock Software available from tensorflow.org.

\bibitem[Alemi et~al.(2017)Alemi, Poole, Fischer, Dillon, Saurous, and
  Murphy]{alemi2017fixing}
A.~A. Alemi, B.~Poole, I.~Fischer, J.~V. Dillon, R.~A. Saurous, and K.~Murphy.
\newblock Fixing a broken {ELBO}.
\newblock \emph{arXiv preprint arXiv:1711.00464}, 2017.

\bibitem[Atchison and Shen(1980)]{atchison1980logistic}
J.~Atchison and S.~M. Shen.
\newblock Logistic-normal distributions: Some properties and uses.
\newblock \emph{Biometrika}, 67\penalty0 (2):\penalty0 261--272, 1980.

\bibitem[Baldi and Hornik(1989)]{baldi1989neural}
P.~Baldi and K.~Hornik.
\newblock Neural networks and principal component analysis: Learning from
  examples without local minima.
\newblock \emph{Neural networks}, 2\penalty0 (1):\penalty0 53--58, 1989.

\bibitem[Bartholomew(1987)]{bartholomew1987latent}
D.~J. Bartholomew.
\newblock \emph{Latent variable models and factors analysis}.
\newblock Oxford University Press, Inc., 1987.

\bibitem[Blei et~al.(2017)Blei, Kucukelbir, and McAuliffe]{blei2017variational}
D.~M. Blei, A.~Kucukelbir, and J.~D. McAuliffe.
\newblock Variational inference: A review for statisticians.
\newblock \emph{Journal of the American Statistical Association}, 2017.

\bibitem[Bowman et~al.(2015)Bowman, Vilnis, Vinyals, Dai, Jozefowicz, and
  Bengio]{bowman2015generating}
S.~R. Bowman, L.~Vilnis, O.~Vinyals, A.~M. Dai, R.~Jozefowicz, and S.~Bengio.
\newblock Generating sentences from a continuous space.
\newblock \emph{arXiv preprint arXiv:1511.06349}, 2015.

\bibitem[Cand{\`e}s et~al.(2011)Cand{\`e}s, Li, Ma, and
  Wright]{candes2011robust}
E.~J. Cand{\`e}s, X.~Li, Y.~Ma, and J.~Wright.
\newblock Robust principal component analysis?
\newblock \emph{Journal of the ACM (JACM)}, 58\penalty0 (3):\penalty0 11, 2011.

\bibitem[Chechik et~al.(2005)Chechik, Globerson, Tishby, and
  Weiss]{chechik2005information}
G.~Chechik, A.~Globerson, N.~Tishby, and Y.~Weiss.
\newblock Information bottleneck for gaussian variables.
\newblock \emph{Journal of machine learning research}, 6\penalty0
  (Jan):\penalty0 165--188, 2005.

\bibitem[Chen et~al.(2018)Chen, Li, Grosse, and Duvenaud]{chen2018isolating}
R.~T.~Q. Chen, X.~Li, R.~Grosse, and D.~Duvenaud.
\newblock Isolating sources of disentanglement in variational autoencoders.
\newblock \emph{Advances in Neural Information Processing Systems}, 2018.

\bibitem[Chen et~al.(2016)Chen, Kingma, Salimans, Duan, Dhariwal, Schulman,
  Sutskever, and Abbeel]{chen2016variational}
X.~Chen, D.~P. Kingma, T.~Salimans, Y.~Duan, P.~Dhariwal, J.~Schulman,
  I.~Sutskever, and P.~Abbeel.
\newblock Variational lossy autoencoder.
\newblock \emph{arXiv preprint arXiv:1611.02731}, 2016.

\bibitem[Cremer et~al.(2018)Cremer, Li, and Duvenaud]{cremer2018inference}
C.~Cremer, X.~Li, and D.~Duvenaud.
\newblock Inference suboptimality in variational autoencoders.
\newblock \emph{arXiv preprint arXiv:1801.03558}, 2018.

\bibitem[Dai and Wipf(2019)]{dai2018diagnosing}
B.~Dai and D.~Wipf.
\newblock Diagnosing and enhancing {VAE} models.
\newblock In \emph{International Conference on Learning Representations}, 2019.

\bibitem[Dai et~al.(2017)Dai, Wang, Aston, Hua, and Wipf]{dai2017hidden}
B.~Dai, Y.~Wang, J.~Aston, G.~Hua, and D.~Wipf.
\newblock Hidden talents of the variational autoencoder.
\newblock \emph{arXiv preprint arXiv:1706.05148}, 2017.

\bibitem[Dieng et~al.(2018)Dieng, Kim, Rush, and Blei]{dieng2018avoiding}
A.~B. Dieng, Y.~Kim, A.~M. Rush, and D.~M. Blei.
\newblock Avoiding latent variable collapse with generative skip models.
\newblock \emph{arXiv preprint arXiv:1807.04863}, 2018.

\bibitem[Gomez-Bombarelli et~al.(2018)Gomez-Bombarelli, Wei, Duvenaud,
  Hernandez-Lobato, Sanchez-Lengeling, Sheberla, Aguilera-Iparraguirre, Hirzel,
  Adams, and Aspuru-Guzik]{molauto18}
R.~Gomez-Bombarelli, J.~N. Wei, D.~Duvenaud, J.~M. Hernandez-Lobato,
  B.~Sanchez-Lengeling, D.~Sheberla, J.~Aguilera-Iparraguirre, T.~D. Hirzel,
  R.~P. Adams, and A.~Aspuru-Guzik.
\newblock Automatic chemical design using a data-driven continuous
  representation of molecules.
\newblock \emph{American Chemical Society Central Science}, 2018.

\bibitem[He et~al.(2019)He, Spokoyny, Neubig, and
  Berg-Kirkpatrick]{he2018lagging}
J.~He, D.~Spokoyny, G.~Neubig, and T.~Berg-Kirkpatrick.
\newblock Lagging inference networks and posterior collapse in variational
  autoencoders.
\newblock In \emph{International Conference on Learning Representations}, 2019.

\bibitem[Higgins et~al.(2016)Higgins, Matthey, Pal, Burgess, Glorot, Botvinick,
  Mohamed, and Lerchner]{higgins2016beta}
I.~Higgins, L.~Matthey, A.~Pal, C.~Burgess, X.~Glorot, M.~Botvinick,
  S.~Mohamed, and A.~Lerchner.
\newblock beta-{VAE}: Learning basic visual concepts with a constrained
  variational framework.
\newblock In \emph{International Conference on Learning Representations}, 2016.

\bibitem[Hjelm et~al.(2016)Hjelm, Salakhutdinov, Cho, Jojic, Calhoun, and
  Chung]{hjelm2016iterative}
D.~Hjelm, R.~R. Salakhutdinov, K.~Cho, N.~Jojic, V.~Calhoun, and J.~Chung.
\newblock Iterative refinement of the approximate posterior for directed belief
  networks.
\newblock In \emph{Advances in Neural Information Processing Systems}, 2016.

\bibitem[Huang et~al.(2018)Huang, Tan, Lacoste, and
  Courville]{huang2018improving}
C.-W. Huang, S.~Tan, A.~Lacoste, and A.~C. Courville.
\newblock Improving explorability in variational inference with annealed
  variational objectives.
\newblock In \emph{Advances in Neural Information Processing Systems}, 2018.

\bibitem[Jordan et~al.(1999)Jordan, Ghahramani, Jaakkola, and
  Saul]{jordan1999introduction}
M.~I. Jordan, Z.~Ghahramani, T.~S. Jaakkola, and L.~K. Saul.
\newblock An introduction to variational methods for graphical models.
\newblock \emph{Machine learning}, 1999.

\bibitem[Kim et~al.(2018)Kim, Wiseman, Miller, Sontag, and Rush]{kim2018semi}
Y.~Kim, S.~Wiseman, A.~C. Miller, D.~Sontag, and A.~M. Rush.
\newblock Semi-amortized variational autoencoders.
\newblock \emph{arXiv preprint arXiv:1802.02550}, 2018.

\bibitem[Kingma and Ba(2014)]{kingma2014adam}
D.~P. Kingma and J.~Ba.
\newblock Adam: A method for stochastic optimization.
\newblock \emph{arXiv preprint arXiv:1412.6980}, 2014.

\bibitem[Kingma and Welling(2013)]{kingma2013auto}
D.~P. Kingma and M.~Welling.
\newblock Auto-encoding variational bayes.
\newblock \emph{arXiv preprint arXiv:1312.6114}, 2013.

\bibitem[Kingma et~al.(2016)Kingma, Salimans, Jozefowicz, Chen, Sutskever, and
  Welling]{kingma2016improved}
D.~P. Kingma, T.~Salimans, R.~Jozefowicz, X.~Chen, I.~Sutskever, and
  M.~Welling.
\newblock Improved variational inference with inverse autoregressive flow.
\newblock In \emph{Advances in neural information processing systems}, pages
  4743--4751, 2016.

\bibitem[Kunin et~al.(2019)Kunin, Bloom, Goeva, and Seed]{kunin2019loss}
D.~Kunin, J.~M. Bloom, A.~Goeva, and C.~Seed.
\newblock Loss landscapes of regularized linear autoencoders.
\newblock \emph{arXiv preprint arXiv:1901.08168}, 2019.

\bibitem[LeCun(1998)]{lecun1998mnist}
Y.~LeCun.
\newblock The mnist database of handwritten digits.
\newblock \emph{http://yann. lecun. com/exdb/mnist/}, 1998.

\bibitem[Liu et~al.(2015)Liu, Luo, Wang, and Tang]{liu2015faceattributes}
Z.~Liu, P.~Luo, X.~Wang, and X.~Tang.
\newblock Deep learning face attributes in the wild.
\newblock In \emph{Proceedings of International Conference on Computer Vision
  (ICCV)}, 2015.

\bibitem[Ma et~al.(2019)Ma, Zhou, and Hovy]{ma2018mae}
X.~Ma, C.~Zhou, and E.~Hovy.
\newblock {MAE}: Mutual posterior-divergence regularization for variational
  autoencoders.
\newblock In \emph{International Conference on Learning Representations}, 2019.

\bibitem[Maal{\o}e et~al.(2019)Maal{\o}e, Fraccaro, Li{\'e}vin, and
  Winther]{maaloe2019biva}
L.~Maal{\o}e, M.~Fraccaro, V.~Li{\'e}vin, and O.~Winther.
\newblock {BIVA}: A very deep hierarchy of latent variables for generative
  modeling.
\newblock \emph{arXiv preprint arXiv:1902.02102}, 2019.

\bibitem[Papamakarios et~al.(2017)Papamakarios, Pavlakou, and
  Murray]{NIPS2017_6828}
G.~Papamakarios, T.~Pavlakou, and I.~Murray.
\newblock Masked autoregressive flow for density estimation.
\newblock In \emph{Advances in Neural Information Processing Systems}. 2017.

\bibitem[Petersen et~al.()]{petersen2008matrix}
K.~B. Petersen et~al.
\newblock The matrix cookbook.

\bibitem[Razavi et~al.(2019)Razavi, van~den Oord, Poole, and
  Vinyals]{razavi2018preventing}
A.~Razavi, A.~van~den Oord, B.~Poole, and O.~Vinyals.
\newblock Preventing posterior collapse with delta-{VAE}s.
\newblock In \emph{International Conference on Learning Representations}, 2019.

\bibitem[Rezende and Viola(2018)]{rezende2018taming}
D.~J. Rezende and F.~Viola.
\newblock Taming {VAE}s.
\newblock \emph{arXiv preprint arXiv:1810.00597}, 2018.

\bibitem[Rezende et~al.(2014)Rezende, Mohamed, and
  Wierstra]{rezende2014stochastic}
D.~J. Rezende, S.~Mohamed, and D.~Wierstra.
\newblock Stochastic backpropagation and approximate inference in deep
  generative models.
\newblock \emph{arXiv preprint arXiv:1401.4082}, 2014.

\bibitem[Rolinek et~al.(2018)Rolinek, Zietlow, and
  Martius]{rolinek2018variational}
M.~Rolinek, D.~Zietlow, and G.~Martius.
\newblock Variational autoencoders pursue {PCA} directions (by accident).
\newblock \emph{arXiv preprint arXiv:1812.06775}, 2018.

\bibitem[Rumelhart et~al.(1985)Rumelhart, Hinton, and
  Williams]{rumelhart1985learning}
D.~E. Rumelhart, G.~E. Hinton, and R.~J. Williams.
\newblock Learning internal representations by error propagation.
\newblock Technical report, California Univ San Diego La Jolla Inst for
  Cognitive Science, 1985.

\bibitem[S{\o}nderby et~al.(2016)S{\o}nderby, Raiko, Maal{\o}e, S{\o}nderby,
  and Winther]{sonderby2016ladder}
C.~K. S{\o}nderby, T.~Raiko, L.~Maal{\o}e, S.~K. S{\o}nderby, and O.~Winther.
\newblock Ladder variational autoencoders.
\newblock In \emph{Advances in neural information processing systems}, pages
  3738--3746, 2016.

\bibitem[Tipping and Bishop(1999)]{tipping1999probabilistic}
M.~E. Tipping and C.~M. Bishop.
\newblock Probabilistic principal component analysis.
\newblock \emph{Journal of the Royal Statistical Society: Series B (Statistical
  Methodology)}, 61\penalty0 (3):\penalty0 611--622, 1999.

\bibitem[Tomczak and Welling(2017)]{tomczak2017vae}
J.~M. Tomczak and M.~Welling.
\newblock Vae with a {V}amp{P}rior.
\newblock \emph{arXiv preprint arXiv:1705.07120}, 2017.

\bibitem[Yeung et~al.(2017)Yeung, Kannan, Dauphin, and
  Fei-Fei]{yeung2017tackling}
S.~Yeung, A.~Kannan, Y.~Dauphin, and L.~Fei-Fei.
\newblock Tackling over-pruning in variational autoencoders.
\newblock \emph{arXiv preprint arXiv:1706.03643}, 2017.

\end{thebibliography}
